\documentclass[lettersize,journal]{IEEEtran}
\usepackage{amsmath,amsfonts}
\usepackage{algorithmic}
\usepackage{algorithm}
\usepackage{array}
\usepackage[caption=false,font=normalsize,labelfont=sf,textfont=sf]{subfig}
\usepackage{textcomp}
\usepackage{stfloats}
\usepackage{url}
\usepackage{verbatim}
\usepackage{graphicx}
\usepackage{cite}
\usepackage{booktabs}       
\usepackage{amsfonts}       
\usepackage{nicefrac}       
\usepackage{microtype}      
\usepackage[table]{xcolor}
\usepackage{amsmath,amsthm}
\usepackage{graphicx}
\usepackage{algorithm,algorithmic}

\usepackage{multirow}
\usepackage{arydshln}
\usepackage{siunitx} 

\usepackage{float}

\usepackage{hyperref}
\newtheorem{Proposition}{Proposition}
\newtheorem{Definition}{Definition}
\usepackage{tcolorbox}
\usepackage{pifont}
\definecolor{mygreen}{HTML}{3FBC9D}
\newcommand{\cmark}{{\color{mygreen}\ding{51}}}
\newcommand{\xmark}{{\color{red}\ding{55}}}

\begin{document}

\title{Closing the Oracle Gap: Increment Vector Transformation for\\ Class Incremental Learning}

\author{Zihuan Qiu, Yi Xu~\IEEEmembership{Member,~IEEE}, Fanman Meng~\IEEEmembership{Member,~IEEE}, Runtong Zhang, \\Linfeng Xu~\IEEEmembership{Member,~IEEE}, Qingbo Wu~\IEEEmembership{Member,~IEEE}, Hongliang Li,~\IEEEmembership{Senior Member,~IEEE}
\thanks{
Zihuan Qiu, Fanman Meng, Runtong Zhang, Linfeng Xu, Qingbo Wu and Hongliang Li are with the University of Electronic
Science and Technology of China, China. E-mail:
\{zihuanqiu@std., fmmeng@ rtzhang@std., lfxu@, qbwu@, hlli@\}uestc.edu.cn.
Yi Xu is with the Dalian University of Technology, China. E-mail: yxu@dlut.edu.cn.}
}

\markboth{Journal of \LaTeX\ Class Files,~Vol.~14, No.~8, August~2021}%
{Shell \MakeLowercase{\textit{et al.}}: A Sample Article Using IEEEtran.cls for IEEE Journals}


\maketitle

\begin{abstract}
Class Incremental Learning (CIL) aims to sequentially acquire knowledge of new classes without forgetting previously learned ones. Despite recent progress, current CIL methods still exhibit significant performance gaps compared to their oracle counterparts—models trained with full access to historical data. Inspired by recent insights on Linear Mode Connectivity (LMC), we revisit the geometric properties of oracle solutions in CIL and uncover a fundamental observation: these oracle solutions typically maintain low-loss linear connections to the optimum of previous tasks.
Motivated by this finding, we propose Increment Vector Transformation (IVT), a novel plug-and-play framework designed to mitigate catastrophic forgetting during training. Rather than directly following CIL updates, IVT periodically teleports the model parameters to transformed solutions that preserve linear connectivity to previous task optimum. By maintaining low-loss along these connecting paths, IVT effectively ensures stable performance on previously learned tasks. The transformation is efficiently approximated using diagonal Fisher Information Matrices, making IVT suitable for both exemplar-free and exemplar-based scenarios, and compatible with various initialization strategies.
Extensive experiments on CIFAR-100, FGVCAircraft, ImageNet-Subset, and ImageNet-Full demonstrate that IVT consistently enhances the performance of strong CIL baselines. Specifically, on CIFAR-100, IVT improves the last accuracy of the PASS baseline by +5.12\% and reduces forgetting by 2.54\%. For the CLIP-pre-trained SLCA baseline on FGVCAircraft, IVT yields gains of +14.93\% in average accuracy and +21.95\% in last accuracy. The code will be released.
\end{abstract}


\section{Introduction}
\IEEEPARstart{C}{lass} Incremental Learning (CIL) poses a significant challenge in machine learning, requiring models to learn sequentially without access to previous training data. A notorious phenomenon in this paradigm is catastrophic forgetting \cite{mccloskey1989catastrophic}, where models overwrite previously acquired knowledge when adapting to new tasks.
To mitigate this, various approaches have been proposed. \textit{Regularization methods} \cite{Kirkpatrick2016OvercomingCF,zenke2017continual} constrain updates to crucial parameters for past tasks or 
transfer knowledge from previous tasks through intermediate features and outputs \cite{Kirkpatrick2016OvercomingCF,Hou2019LearningAU,Douillard2020PODNetPO}. \textit{Memory replay methods} \cite{Rebuffi2016iCaRLIC,liu2020mnemonics,luo2023class} retain a subset of exemplars from previous tasks for rehearsal, selecting representative samples to optimize memory efficiency.
\textit{Dynamic architecture methods} \cite{liu2021adaptive,zhou2022model} introduce new network components to accommodate new tasks. However, despite these advancements, incremental models still fall short compared to oracles that \textit{trained incrementally with access to all historical data}.

\begin{figure}[t]
\centering
  \includegraphics[width=.95\linewidth]{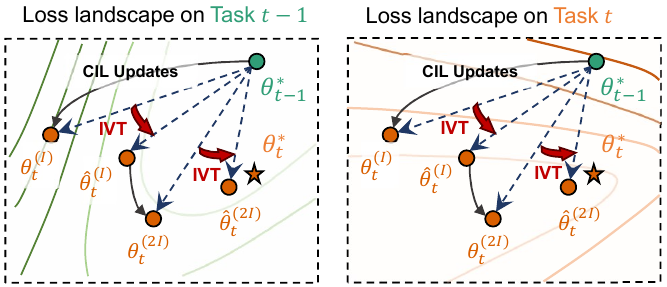}
\caption{
\textbf{Illustration of Increment Vector Transformation (IVT).} 
In class incremental learning, standard updates often drift away from the low-loss region of previous tasks, leading to forgetting. IVT mitigates this by periodically teleporting the model parameters $\theta_t$ toward regions of low loss for both past and current tasks (shallower contours). This is done every $I$ steps by transforming (red arrows) the increment vector (blue dashed line)—the parameter displacement between the previous optimum $\theta_{t-1}^*$ and the current point $\theta_t^{(m)}$. The transformed update brings $\theta_t$ closer to the oracle solution $\theta_t^*$, enabling better retention of prior knowledge while learning new classes.
}
  \label{fig:intro}
\end{figure} 

Recent studies on mode connectivity in neural networks have shed new light on the persistent performance gap in CIL \cite{draxler2018essentially, garipov2018loss, frankle2020linear}. Mode connectivity refers to the existence of low-loss paths connecting different minima in the loss landscape. While prior work by Wen \textit{et al.} \cite{wen2023optimizing} reveals that many advanced CIL methods lack such favorable linear mode connectivity (LMC), suggesting that their incremental solutions are not linearly connected without suffering high loss. However, this property remains poorly understood in oracle settings.
Although Mirzadeh \textit{et al.} \cite{Mirzadeh2020LinearMC} demonstrated that LMC can exist under a Naive-SGD oracle (\textit{i.e.}, full replay with vanilla training), their analysis is limited to basic setups and does not extend to more realistic, constraint-aware CIL formulations.

In this paper, we pose a key open question: Does LMC hold for the oracle of advanced CIL methods? We define a CIL oracle as a version of a given CIL algorithm that has access to full past data, but retains its structural components (\textit{e.g}., distillation). If such CIL oracles preserve LMC, it opens a promising direction: even strong CIL algorithms could, in principle, be guided along low-loss paths—offering a new avenue for mitigating forgetting without violating their algorithmic constraints.

Motivated by this question, we investigate whether LMC holds for CIL oracles, and empirically confirm its presence across multiple configurations. Our results show that traversing these low-loss paths allows models to retain strong performance on earlier tasks, while simultaneously integrating new knowledge. In fact, CIL oracles appear to reside on well-connected low-loss manifolds, even under CIL algorithmic constraints such as distillation or regularization.

These insights motivate us to explicitly guide optimization trajectories toward low-loss manifolds during training. To this end, we propose \textbf{I}ncrement \textbf{V}ector \textbf{T}ransformation (IVT), a plug-and-play framework that analyzes the deviation between incremental solutions and oracle behavior through the lens of their increment vectors. Specifically, we define the increment vector $V_t$ as the directional update from the previous task optimum $\theta_{t-1}^*$ to the current solution $\theta_t$, and construct a transformation that maps $V_t$ to a low-loss counterpart $\hat{V}_t$. This transformation leverages curvature information—efficiently approximated via the diagonal Fisher Information Matrix—to preserve task-relevant directions and promote low-loss connectivity across tasks. As illustrated in Figure~\ref{fig:intro}, IVT enables the model to reach a solution closer to the oracle $\theta_t^*$, striking a more favorable trade-off between stability and plasticity.
Unlike prior works \cite{Mirzadeh2020LinearMC,wen2023optimizing,marczak2024magmax,marouf2024weightedensemblemodelsstrong} that passively identify favorable models along post-hoc mode connectivity paths or depend on exemplar memory, IVT operates \textit{during training}, actively shaping the optimization trajectory \textit{without} relying on replay data. This makes IVT broadly applicable to both exemplar-free and exemplar-based settings, and compatible with diverse training paradigms—ranging from scratch training to large-scale pre-training (\textit{e.g.}, CLIP). 
Extensive experiments on CIFAR-100, FGVCAircraft, ImageNet-Subset, and ImageNet-Full confirm that IVT significantly improves performance when integrated into representative CIL methods.

Our contributions are summarized as follows:
\begin{enumerate} 
\item We present the analysis of linear mode connectivity (LMC) in \emph{CIL oracles}. While prior work has examined LMC in standard CIL or naive SGD settings, we are the first to investigate whether advanced CIL formulations possess LMC when given oracle access. Our empirical results confirm that CIL oracles exhibit favorable LMC, supporting stable knowledge integration through low-loss paths.
\item We provide a theoretical analysis of the mismatch in incremental models, revealing a misalignment between naive update directions and oracle solutions. Based on this insight, we propose Increment Vector Transformation (IVT), a novel method that dynamically transforms the optimization toward low-loss regions using curvature information.
\item IVT leverages the diagonal Fisher Information Matrix to approximate the Hessian transformation efficiently, enabling practical application in large-scale networks. It is compatible with both exemplar-based and exemplar-free CIL settings and supports models trained from scratch as well as from large-scale pre-training.
\item We integrate IVT into multiple representative CIL baselines and evaluate on CIFAR-100, FGVCAircraft, ImageNet-Subset, and ImageNet-Full. Experiments consistently demonstrate substantial improvements in accuracy and forgetting metrics across different training paradigms, validating the generality and effectiveness of our approach.
\end{enumerate}

\section{Related work}

\paragraph{Class Incremental Learning}
Existing CIL methods can be broadly categorized into three main approaches. 
\textit{Regularization methods} mitigate catastrophic forgetting by imposing constraints on model parameters or outputs. Approaches like EWC \cite{Kirkpatrick2016OvercomingCF} calculate the importance of parameters for previous tasks and penalize changes to crucial parameters, while knowledge distillation techniques such as LUCIR \cite{Hou2019LearningAU}, PODNet \cite{Douillard2020PODNetPO}, and GeoDL \cite{simon2021learning} use output logits or intermediate features to preserve learned representations. To address class imbalance, methods like BiC \cite{wu2019large} and FOSTER \cite{wang2022foster} apply post-hoc corrections and classifier adjustments to reduce bias toward newly introduced classes. GR \cite{he2024gradient} reweights gradients to balance long-tailed optimization.

\textit{Memory replay methods} store a subset of exemplars and replay them during new task learning. For instance, iCaRL \cite{Rebuffi2016iCaRLIC} selects samples that best approximate class means. Memory management strategy  \cite{liu2020mnemonics,liu2021rmm,luo2023class,liu2023online} optimize exemplar selection or compression to maximize memory efficiency. 
When storing real data is infeasible due to privacy or memory constraints, prompt-based methods \cite{wang2022learning,wang2022dualprompt,wang2022s,smith2023coda}, prototype-based approaches \cite{zhu2021prototype,zhu2022self,magistri2024elastic,li2024fcs}, synthetic techniques \cite{Smith2021AlwaysBD, choi2021dual,gao2022r,qiu2024dual}, and analytic approaches \cite{zhuang2022acil,zhuang2023gkeal,zhuang2024ds}
avoiding replay to satisfy these constraints. 

\textit{Dynamic architecture methods} adapt the network structure to accommodate new tasks by expanding or modifying network components. Expansion-based approaches \cite{liu2021adaptive,yan2021dynamically,zhou2022model,rypesc2024divide,zhou2024expandable} dynamically allocate resources, effectively isolating new knowledge from previously acquired information. Ensemble-based method \cite{lee2017overcoming,sun2024incremental,marouf2024weightedensemblemodelsstrong, marczak2024magmax} adaptively balance stability and plasticity, allowing the model to learn new information flexibly while preserving existing knowledge.

\paragraph{Mode Connectivity}
Mode connectivity is a phenomenon where different minima in the loss landscape of deep neural networks are connected by low-loss paths in the parameter space \cite{draxler2018essentially,garipov2018loss}. It offers a novel perspective on optimization, suggesting that optimum obtained through gradient-based methods are points on a connected, low-loss manifold. 
Various methods, such as polygonal chains, Bézier curves, elastic bands, and simplicial complexes, have been used to model these low-loss paths \cite{draxler2018essentially,garipov2018loss,benton2021loss}. The initialization of minima plays a crucial role: high-loss ridge often exists along the linear path between minima trained from different initializations, but linear connectivity can be achieved when minima share the same initialization and are stable to SGD noise \cite{frankle2020linear,neyshabur2020being}.
Recently, mode connectivity has found applications in CIL. By leveraging the geometric structure of the loss landscape, recent methods merge models along low-loss paths, enabling task composition without revisiting prior data \cite{jin2022dataless, liu2023tangent, marczak2024magmax, marouf2024weightedensemblemodelsstrong, tang2025merging,qiu2025mingle}. This model merging paradigm aligns closely with mode connectivity, as both rely on the existence of shared low-loss regions that allow knowledge to be integrated with minimal interference. This connection opens promising directions for scalable, data-free continual learning through geometric reasoning in parameter space.

\textbf{Discussion.} 
Recent studies have explored the connection between LMC and CIL, offering insights that motivate our work. Mirzadeh \textit{et al.}~\cite{Mirzadeh2020LinearMC} showed that Naive-SGD can exhibit LMC under full replay, but their analysis excludes regularization or distillation strategies used in modern CIL. Wen \textit{et al.}~\cite{wen2023optimizing} extended this line of inquiry, revealing that while advanced CIL methods achieve stronger performance, they often fail to preserve LMC. In contrast, we show that although individual incremental solutions may not lie on a connected manifold, their oracle counterparts do exhibit LMC. This observation motivates our method, which dynamically steers the optimization trajectory toward the low-loss linear regions that characterize such oracle solutions.

Methodologically, our approach differs in several key aspects. EOPC~\cite{wen2023optimizing} identifies low-forgetting checkpoints via interpolation using exemplars, whereas IVT directly adjusts the update direction during training, without requiring replay memory. While EOPC is limited to from-scratch training, IVT supports both randomly initialized and pre-trained models, \textit{e.g.}, CLIP~\cite{radford2021learning}. CoFiMA~\cite{marouf2024weightedensemblemodelsstrong} and MagMax~\cite{marczak2024magmax} also exploit LMC, but in a post-hoc fashion without shaping learning trajectories. Moreover, in our experiments, both methods show significant performance drops when trained from scratch, highlighting their strong dependence on pre-training. In contrast, IVT consistently improves performance across diverse CIL regimes, emphasizing the benefits of online trajectory shaping and broader robustness to training conditions.

\section{Revisiting Linear Mode Connectivity in CIL}
\label{sec2}
The forgetting analysis based on Taylor expansion is commonly used in CIL \cite{yin2020optimization,mirzadeh2020understanding,wu2024meta}. For simplicity, suppose that there are two tasks, $\mathcal{T}_1$ and $\mathcal{T}_2$. Let $\theta_1$ be the minima obtained on $\mathcal{T}_1$, we perform a second-order Taylor expansion of $\mathcal{L}_1(\theta)$ at $\theta_1$:
\begin{align}
    \mathcal{L}_{1}(\theta) \approx& \mathcal{L}_{1}\left( \theta_1 \right)+\left(\theta-\theta_1\right)^{\top} \nabla \mathcal{L}_{1}\left(\theta_1\right) \nonumber \\ 
    &+\frac{1}{2}\left(\theta-\theta_1\right)^{\top} H_1 \left(\theta-\theta_1\right)\\
    \approx& \mathcal{L}_{1}\left( \theta_1 \right) + \frac{1}{2}\left(\theta-\theta_1\right)^{\top} H_1 \left(\theta-\theta_1\right).\label{eq2}
\end{align}
The last equality holds because, at the minima $\theta_1$ of $\mathcal{T}_1$, the model is assumed to converge and thus  $ \nabla \mathcal{L}_{1}\left(\theta_1\right) \approx 0$. Besides, the Hessian matrix $H_1 = \nabla^{2} \mathcal{L}_{1}\left(\theta_1\right)$ needs to be positive semi-definite at the converged minima. Therefore, the forgetting $\mathcal{F}_1$ can be bounded as follows:
\begin{align}
      \mathcal{F}_1 &= \mathcal{L}_{1}(\theta) - \mathcal{L}_{1}\left( \theta_1 \right) \\
      & \approx \frac{1}{2}\left(\theta-\theta_1\right)^{\top}  H_1 \left(\theta-\theta_1\right) \leq \frac{1}{2} \lambda^{1} \|\Delta \theta\|^2,\label{eq3}  
\end{align}
where $\Delta \theta = \theta-\theta_1$ and $\lambda_{1}$ is the maximum eigenvalue of $H_1$.  When $\Delta \theta$ aligns with the eigenvector corresponding to $\lambda^{1}$, $\mathcal{F}_1$ reaches its upper bound, and the model update follows the direction of maximum curvature of $H_1$. 
Conversely, reducing $\mathcal{F}_1$ can be achieved by minimizing $\Delta \theta$ or by steering the model update direction away from the higher curvature directions of $H_1$.

Recently, some studies have linked catastrophic forgetting in CIL to mode connectivity \cite{mirzadeh2020understanding,verwimp2021rehearsal,wen2023optimizing}. Mirzadeh \textit{et al.} \cite{Mirzadeh2020LinearMC} empirically demonstrate that Naive-SGD oracle obtained through joint training with all previous data lies within the same low-loss region as the solutions for previous tasks and can be connected to them via low-loss linear paths. 
Moving along this path does not significantly impact the performance for previous tasks, suggesting that the Naive-SGD oracle has identified low-curvature directions in the loss landscape for earlier tasks. In contrast, this property does not hold for the incremental solution. Moving along the linear path from the previous solution to the incremental solution often results in a substantial drop in accuracy for previous tasks \cite{mirzadeh2020understanding,wen2023optimizing}.

Beyond the Naive-SGD, we explore the linear mode connectivity (LMC) for the CIL approaches $\theta_{t}$ and its oracle $\theta_{t}^*$, with a particular focus on accuracy consistency and the stability-plasticity trade-off along the linear path. To achieve this, we evaluate the accuracy of a series of interpolation models, starting from the old model $\theta_{i}^*$ (for $i \leq t-1$) and progressing along the linear update direction. Formally, the interpolation models are defined as follows:
\begin{equation}
    \bar{\theta}_{t,i} \left( \lambda \right) = \theta_{i}^* +\lambda U_{t},
\end{equation}
where $U_{t}=\left( \theta_{t} -\theta_{i}^* \right) /\| \theta_{t} -\theta_{i}^* \|_{2}$ is the normalized update vector, and $\lambda$ is the interpolation factor. The normalization of $U_t$ ensures that $\lambda$ directly corresponds to the distance between the parameters, facilitating the comparison of models across different settings.
Similarly, we define the interpolation to the CIL oracle as: 
$\bar{\theta}^*_{t,i} \left( \lambda \right) = \theta_{i}^* +\lambda U^*_{t}$, 
where $U^*_{t}=\left( \theta_{t}^* -\theta_{i}^* \right) /\| \theta_{t}^* -\theta_{i}^* \|_{2}$. 
Notably, setting $\hat{\lambda} = \| \theta_{t} -\theta_{i}^* \|_{2}$ recovers $\theta_{t}=\bar{\theta}_{t,i}(\hat{\lambda})$, while $\hat{\lambda}^*  = \| \theta_{t}^* -\theta_{i}^* \|_{2}$ yields $\theta_{t}^*=\bar{\theta}_{t,i}^*(\hat{\lambda}^*)$.
For parameters that differ between the two interpolated models, \textit{e.g.}, classifier parameters for new classes, we initialize them for $\theta_{i}^*$ before interpolation.

\begin{figure*}[t]
\centering
\subfloat{\includegraphics[width=.245\linewidth]{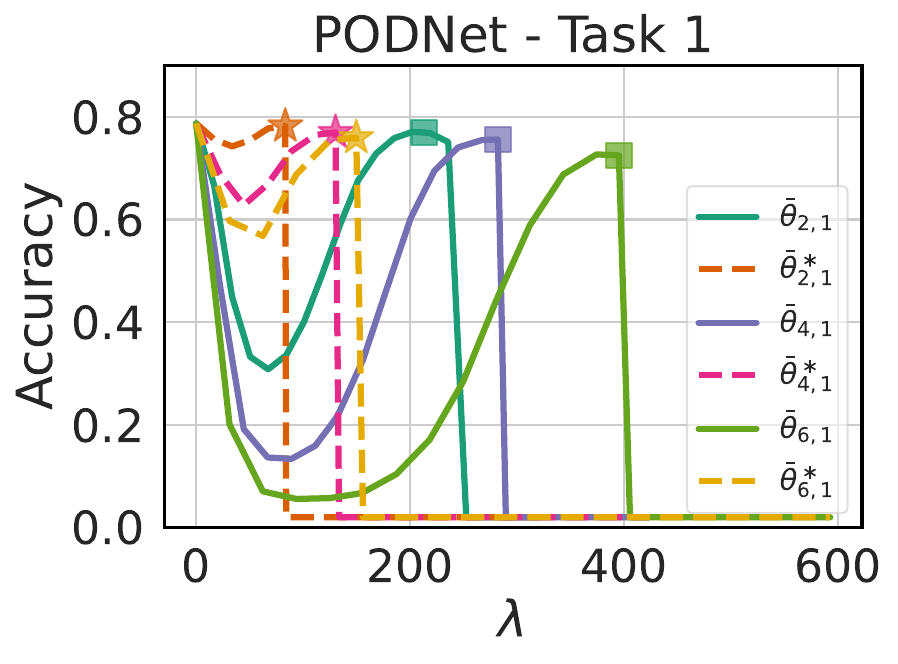}} \hfill
\subfloat{\includegraphics[width=.24\linewidth]{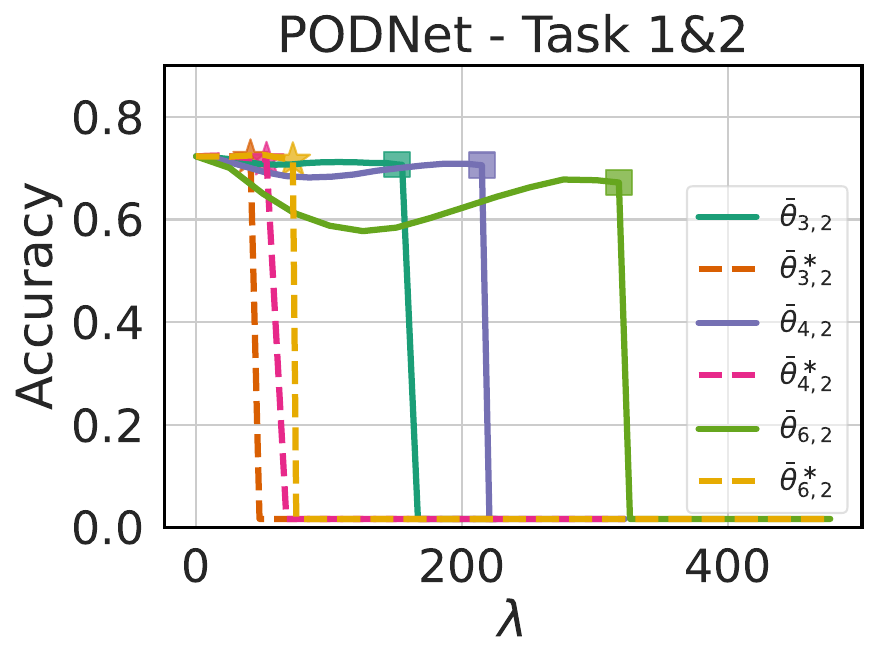}} \hfill
\subfloat{\includegraphics[width=.24\linewidth]{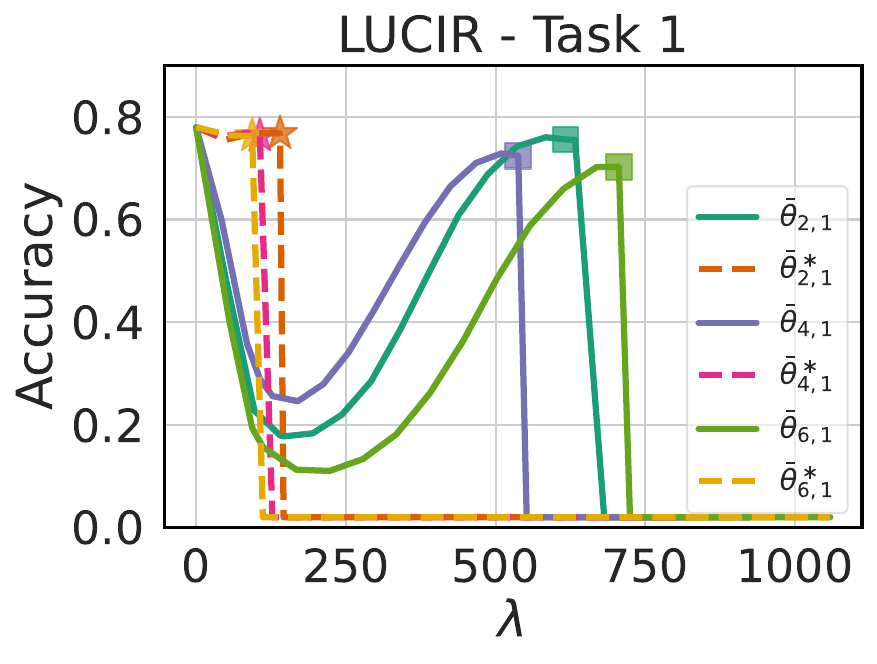}} \hfill
\subfloat{\includegraphics[width=.24\linewidth]{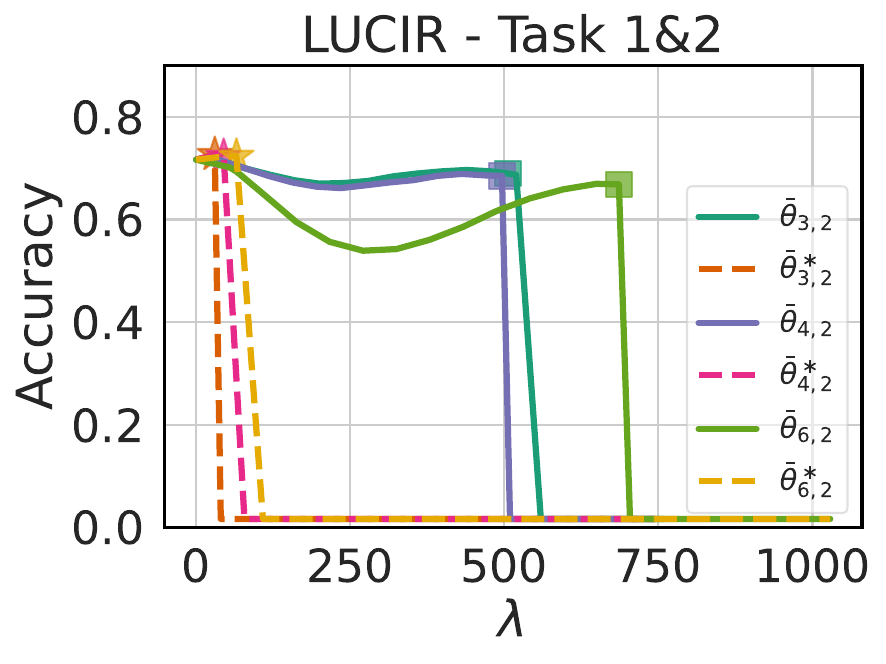}} 
\caption{Evaluating accuracy consistency along the linear path on CIFAR-100 for increments of 5 tasks (\textit{i.e.}, 6 tasks in total). The star and square denote the CIL oracle $\theta_{t}^*=\bar{\theta}_{t,i}^*(\hat{\lambda}^*)$ and the incremental model $\theta_{t}=\bar{\theta}_{t,i}(\hat{\lambda})$. (See Figure  \ref{lmc_2} for comparison when IVT is applied.)} 
\label{lmc_1}
\end{figure*}

\begin{figure*}[t]
\centering
\hspace{10pt}\subfloat{\includegraphics[width=.3\linewidth]{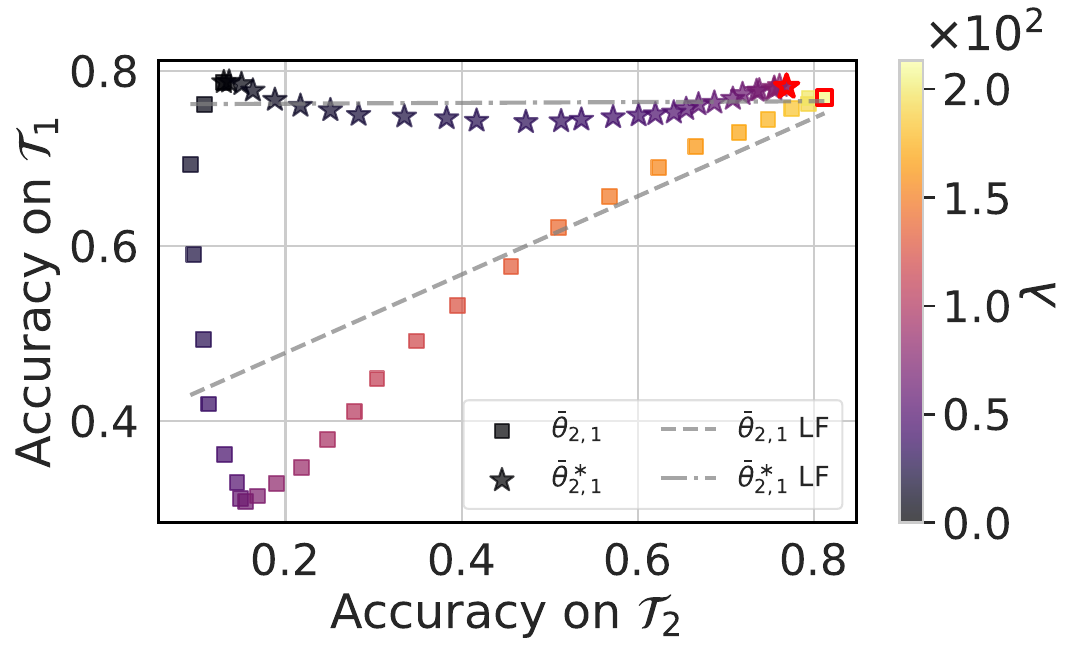}}\hfill
\subfloat{\includegraphics[width=.3\linewidth]{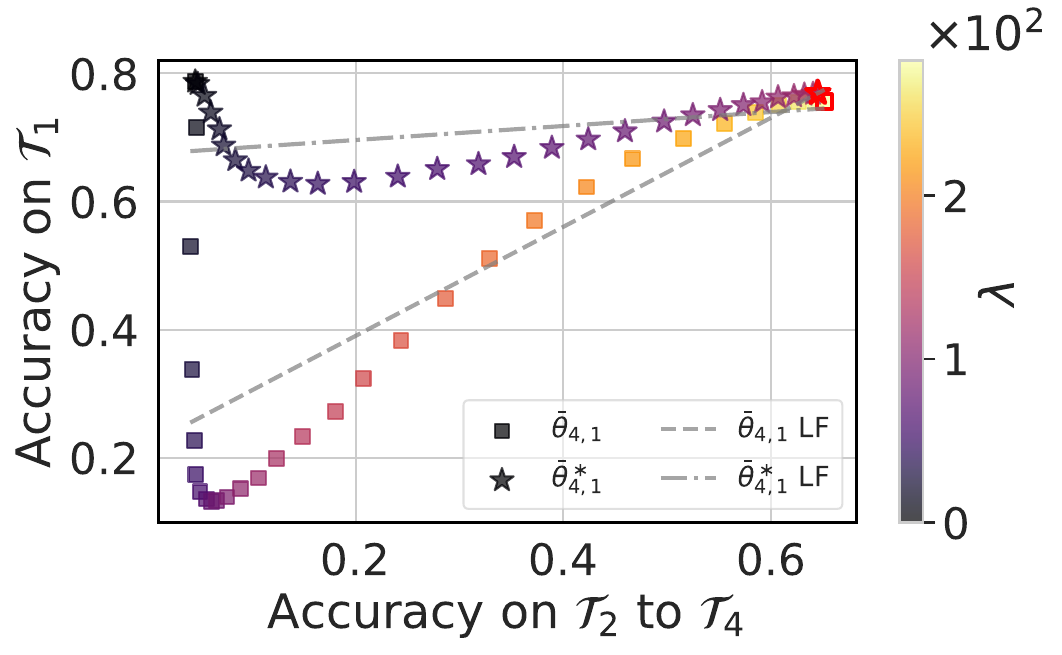}}\hfill
\subfloat{\includegraphics[width=.3\linewidth]{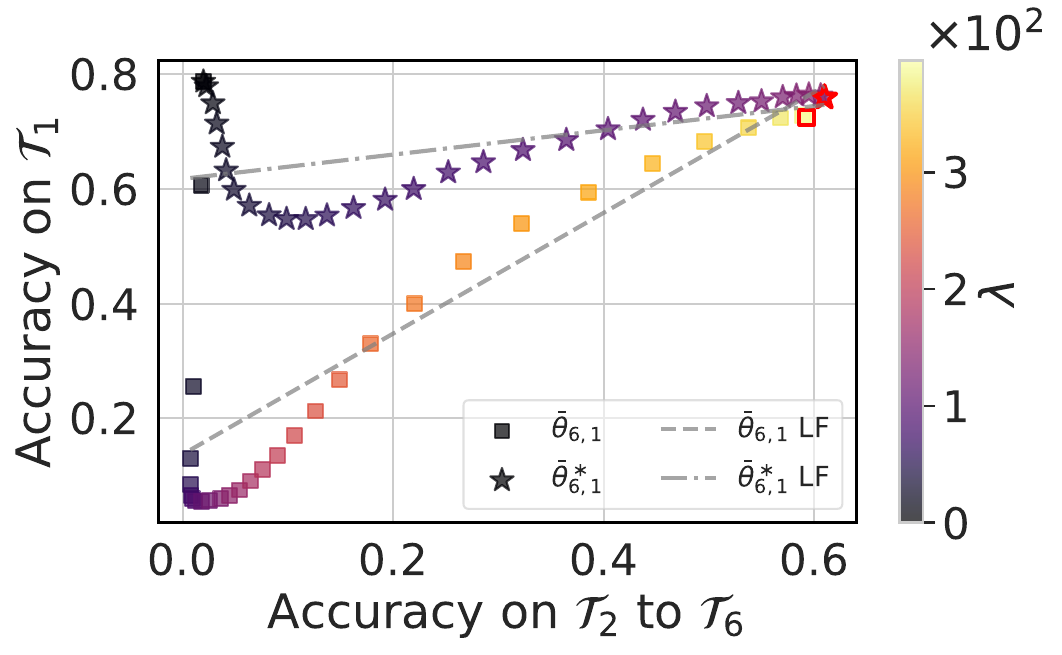}}\hspace{10pt}
\caption{Evaluating stability-plasticity trade-off along the linear path achieved by PODNet on CIFAR-100 for increments of 5 tasks. LF represents the linear fit to the scattered points. The red-edged star and square denote the CIL oracle $\bar{\theta}^*(\hat{\lambda}^*)$ and the incremental model $\bar{\theta}(\hat{\lambda})$, respectively. (See Figure  \ref{blc_2} for comparison when IVT is applied.)} 
\label{blc_1}
\end{figure*}

\subsection{Accuracy Consistency along the Linear Path}
We evaluate accuracy consistency along the linear path on CIFAR-100 using PODNet \cite{Douillard2020PODNetPO} and LUCIR \cite{Hou2019LearningAU}. The experiments consist of an initial task with 50 classes, followed by 5 incremental tasks, each introducing 10 new classes. The incremental model retains 20 exemplars per class, while the CIL oracle has access to the full training data of previous tasks at each incremental step.
Figure  \ref{lmc_1} illustrates the test accuracy of $\mathcal{T}_1$ along the linear path from $\theta_1^*$ to the models of subsequent tasks, as well as the test accuracy of both $\mathcal{T}_1$ and $\mathcal{T}_2$ as we move from $\theta_2^*$ to the models of later tasks.

In Figure  \ref{lmc_1}, we can observe that the CIL oracle achieves better accuracy consistency along the linear path. Concretely, the experiments uncover two key observations:
(1) The CIL oracles tend to remain closer to the minima of previous tasks, suggesting that fully replaying old training data anchors the models near their prior states, leading to a smaller $\Delta \theta$.
(2) The updates of the CIL oracle aligns with the direction of lower curvature. As $\lambda$ increases from 0, the oracle solutions $\bar{\theta}_{t,i}^*$ maintain stable accuracy along the linear path, suggesting they remain within the same low-loss basin as the previous minima, effectively mitigating catastrophic forgetting. In contrast, the accuracy of incremental models $\bar{\theta}_{t,i}$ declines sharply, eventually escaping its previous loss basin and settling into a suboptimal one. This transition is evident from the lower accuracy of $\theta_t = \bar{\theta}_{t,i}(\hat{\lambda})$ compared to $\theta_{t-1}^* = \bar{\theta}_{t,i}(0)$, along with a significant accuracy drop in the middle of the interpolation.

\subsection{Stability-plasticity Trade-off along the Linear Path}
To further investigate the stability-plasticity trade-off of the interpolation models along the linear path, we plot their accuracy on both new and old classes. As depicted in Figure  \ref{blc_1}, we interpolate $\theta_1^*$ with the models of subsequent tasks. 
The figure reveals that as $\lambda$ increases, $\bar{\theta}^*$ and $\bar{\theta}$ exhibit different behaviors. For $\bar{\theta}$, as $\lambda$ increases from 0 to the midpoint, the accuracy on new classes improves while the accuracy on $\mathcal{T}_1$ drops significantly, highlighting a strong stability-plasticity trade-off. As $\lambda$ continues to increase, $\bar{\theta}$ gradually mitigates this trade-off. 
In contrast, $\bar{\theta}^*$ demonstrates a more balanced trade-off, maintaining high performance on $\mathcal{T}_1$ while improving accuracy on new classes. This indicates that $\bar{\theta}^*$ effectively integrates new information without significantly compromising previous knowledge. Such behavior suggests that $\bar{\theta}^*$ resides in a more favorable region of the loss landscape, marked by lower curvature and smoother transitions between tasks, allowing it to achieve better overall performance across both old and new classes as $\lambda$ increases.

\section{Approaching Oracle by Increment Vector Transformation}
The analysis in Section  \ref{sec2} demonstrates the existence of LMC in the CIL oracle. The linear paths discovered by the oracle connect its minima with those of previous tasks while maintaining low error, providing a promising strategy to tackle catastrophic forgetting in CIL. In this section, we aim to approach the oracle by finding these low-loss linear paths.
\begin{Definition}
The increment vector\footnote{Similar to the concept of task vector \cite{ilharco2023editing}, but with the starting point at $\theta_{t-1}$ instead of $\theta_{pre}$.} is defined as the displacement between the model parameters of the previous task and those of the current task.
\end{Definition}
Assuming we start from the same old model\footnote{We can also start from $\theta_{t-1}$, which does not affect the derivation.} $\theta_{t-1}^*$, we can express the oracle $\theta_{t}^*$ and the incremental model $\theta_{t}$ into the sum of $\theta_{t-1}^*$ and their respective increment vectors $V_t^*$ and $V_t$,
\begin{equation}
   \theta_{t}^* = \theta_{t-1}^* + V_t^*, \hspace{10pt} \theta_{t} = \theta_{t-1}^* + V_t,
\end{equation}
where $V_t^* = \theta_{t}^*-\theta_{t-1}^*$ and $V_t = \theta_{t}-\theta_{t-1}^*$. 
Since $V_t^*$ is derived from the joint training with all previous data, obtaining it under the CIL scenario is challenging. However, there exist a transformation $S_t$ such that:
\begin{equation}
   V_t^*=S_t V_t, \hspace{10pt} \theta_{t}^* = \theta_{t-1}^* + S_t V_t.
\end{equation}
In other words, we aim to solve for $S_t$ to transform $V_t$ into $V_t^*$, ensuring that the incremental model resides in the low-loss region for previous tasks and remains close to $\theta_{t-1}^*$, as analyzed in Section  \ref{sec2}.

In what follows, we first theoretically study the mismatch of the incremental model and derive the form of $S_t$. We then introduce a practical method that exploits this spirit with almost no additional training cost.

\subsection{Analyzing the Mismatch of Incremental Model}
Assuming the learning process starts from $\mathcal{T}_{t-1}$, the optimization objective for incremental model $\theta_t$ is to minimize the loss function of $\mathcal{T}_{t}$ while incorporating a regularization term that approximates the implicit proxy loss in various CIL methods \cite{wu2024meta}, with parameter initialized by $\theta_{t-1}^*$,
\begin{equation}
\theta_t = \arg \min_\theta \ \mathcal{L}_t(\theta) + \frac{1}{2} \| \theta - \theta_{t-1}^* \|_{\bar{H}_{t-1}}^2, \label{eq8}
\end{equation}
where 
$\bar{H}_{t-1}= \sum_{i=1}^{t-1}H_i=\sum_{i=1}^{t-1}\nabla^2\mathcal{L}_i(\theta^*_i)$ is the cumulative Hessian for previous tasks, and is independent of $\theta_t$. $\|\Delta\theta\|_{\bar{H}_{t-1}}^2 = \Delta\theta^{\top} \bar{H}_{t-1} \Delta\theta$ measures how  different $\theta$ is from $\theta_{t-1}^*$. 
The objective for $\theta_t^*$ is similar, sharing the same initialization $\theta_{t-1}^*$, but it additionally considers the loss functions of all historical tasks $\mathcal{T}_{i}$ ( for $i \leq t-1$):
\begin{equation}
\theta^*_{t} = \arg \min_\theta \ \sum_{i=1}^{t} \mathcal{L}_{i} (\theta ) + \frac{1}{2} \| \theta - \theta_{t-1}^* \|_{\bar{H}_{t-1}}^2. \label{eq9}
\end{equation}
Notably, in the Equation \ref{eq8} and \ref{eq9}, the second term simulates the anti-forgetting mechanism of CIL methods.  This formulation differs significantly from \cite{Mirzadeh2020LinearMC}, where only Naive-SGD is considered.
Based on these optimization objectives, we can quantify the error between $\theta_t$ and $\theta_t^*$ and derive the form of transformation $S_t$ as presented in Proposition \ref{prop1}. 

\begin{Proposition}
Consider the incremental model $\theta_t$ and oracle $\theta_t^*$, both initialized from the old model $\theta_{t-1}^*$, with optimization objectives defined in Eqs. \ref{eq8} and \ref{eq9}. If $\theta_i$ and $\theta_i^*$ are searched within the neighborhood set $\bigcup_{i=1}^{t-1} \mathcal{N}_i$, where $\mathcal{N}_i = \{\theta : d(\theta, \hat{\theta}_i) < \delta_i\}$, then $\theta_t^*$ can be approximately expressed as the sum of $\theta_{t-1}^*$ and an increment vector $(\theta_t - \theta^*_{t-1})$ transformed by the term $(\bar{H}_{t-1}+\bar{H}_{t})^{-1}\bar{H}_{t}$, which is shown below:
\label{prop1}
\begin{equation}
    \theta^*_{t} \approx \theta^*_{t-1} + (\bar{H}_{t-1}+\bar{H}_{t})^{-1}\bar{H}_{t}(\theta_t - \theta^*_{t-1}). \label{eq10}
\end{equation} 
\end{Proposition}

\begin{proof}
We begin by stating the stationarity conditions for both the incremental model $\theta_t$ and the oracle $\theta_t^*$, which are derived from setting the derivatives of the objectives in Eqs. 8 and 9 to zero:
\begin{align}
\bar{H}_{t-1} (\theta_t - \theta^*_{t-1}) &= - \nabla \mathcal{L}_t(\theta_t), \label{eq13}\\
\bar{H}_{t-1} (\theta^*_{t} - \theta^*_{t-1}) &= - \sum_{i=1}^{t}\nabla \mathcal{L}_{i}(\theta^*_{t}), \label{eq14}
\end{align}
Next, we subtract Equation \ref{eq14} from Equation \ref{eq13}, yielding:
\begin{equation}
   \bar{H}_{t-1} (\theta^*_{t} - \theta_t) = -  \sum_{i=1}^{t-1}\nabla \mathcal{L}_{i}(\theta^*_{t}) - \left[ \nabla \mathcal{L}_{t}(\theta^*_{t}) - \nabla \mathcal{L}_t(\theta_t) \right]. \label{eq15}
\end{equation}
To proceed, we apply a first-order Taylor approximation to approximate the difference between the gradients: \begin{equation}
\nabla \mathcal{L}_{t}(\theta^*_{t}) - \nabla \mathcal{L}_t(\theta_t) = H_t (\theta^*_{t} - \theta_t). \label{eq16}
\end{equation}
Substituting Equation \ref{eq16} into Equation \ref{eq15}, we obtain: 
\begin{equation}
   \bar{H}_{t-1} (\theta^*_{t} - \theta_t) = -  \sum_{i=1}^{t-1}\nabla \mathcal{L}_{i}(\theta^*_{t}) -  H_t (\theta^*_{t} - \theta_t). \label{eq17}
\end{equation}
We then move the term $H_t (\theta^*_{t} - \theta_t)$ to the left-hand side and multiply the entire expression by $\bar{H}_{t}^{-1}$: 
\begin{equation}
   \theta^*_{t} - \theta_t = -\bar{H}_{t}^{-1}  \sum_{i=1}^{t-1}\nabla \mathcal{L}_{i}(\theta^*_{t}) \label{eq18}
\end{equation}
Now, by approximating $\nabla \mathcal{L}_{i}$ for each $i$ as in Equation \ref{eq16}, we express: 
\begin{equation}
\theta^*_{t} - \theta_t = -\bar{H}_{t}^{-1} \sum_{i=1}^{t-1} \left[\nabla \mathcal{L}_{i} (\theta^*_{i}) +  H_i (\theta^*_{t} - \theta^*_i) \right]
\end{equation}
Since the gradient $\nabla \mathcal{L}_{i}$ is close to zero for the converged old model $\theta^*_{i}, i\leq t-1$, it can be neglected in practice, leading to: 
\begin{equation}
\theta^*_{t} - \theta_t \approx -\bar{H}_{t}^{-1} \sum_{i=1}^{t-1} \ H_i (\theta^*_{t} - \theta^*_i) \label{eq20}
\end{equation}
Assuming the parameters are searched within the neighborhood set $\bigcup_{i=1}^{t-1} \mathcal{N}_i$, where $\mathcal{N}_i = \{\theta : d(\theta, \hat{\theta}_i) < \delta_i\}$, we follow the approximation from  \cite{huszar2018note}:
\begin{equation}
    \sum_{i=1}^{j-1} H_i (\theta - \theta_i) \approx ( \sum_{i=1}^{j-1} H_i ) (\theta - \theta_{j-1}) \label{eq21}
\end{equation}
Substituting Equation \ref{eq21} into Equation \ref{eq20} and rearranging with respect to $\theta_t^*$, we recover Proposition 1. 
\begin{equation}
\theta^*_{t} \approx \theta^*_{t-1} + (\bar{H}_{t-1}+\bar{H}_{t})^{-1}\bar{H}_{t}(\theta_t - \theta^*_{t-1})
\end{equation}

\end{proof}

From the results in Equation \ref{eq10}, we have the following observations: 
(1) When $\theta_t$ resides within a relatively flat loss landscape for the old tasks, characterized by a small $\bar{H}_{t-1}$, the approximation indicates that $\theta_t^*$ closely aligns with $\theta_t$. This suggests that the incorporation of new tasks does not significantly disrupt the knowledge acquired from previous tasks.
(2) When $\theta_t$ lies in a region of low curvature for the new task, that is, when $H_{t}$ is small and $\bar{H}_{t}$ is approximately equal to $\bar{H}_{t-1}$, then $\theta_t^*$ can be approximated as the arithmetic mean of $\theta_t$ and $\theta^*_{t-1}$.

\subsection{IVT: Increment Vector Transformation for CIL}
In neural networks with numerous parameters, explicitly computing the full Hessian matrix is often impractical. The Fisher Information Matrix (FIM) \cite{fisher1922mathematical,amari1996neural} is an efficient alternative for Hessian estimation, as it can be directly derived from first-order derivatives,
\begin{equation}
    F_{\theta} = \mathbb{E}_{(x,y)\sim p(x, y)}  \left[ \nabla \log p(y|x) \, \nabla \log p(y|x)^{\top} \right].
\end{equation}
The FIM equals the expected Hessian of the negative log-likelihood \cite{martens2020new}, $F_{\theta}=\mathbb{E}_{x}\left[H_{\theta}\right]$. However, storing FIM still requires $|\theta| \times |\theta|$ memory.
As is common in existing approaches \cite{Kirkpatrick2016OvercomingCF,matena2022merging,daheim2023model}, we can reduce the computation cost by considering only its diagonal elements, bringing it to a level comparable to training on $N$ samples. The diagonal of the FIM is computed as follows,
\begin{equation}
    F_t = \mathbb{E}_{(x, y) \in \mathcal{T}_t} \left(\nabla \mathcal{L}_{t}(x, y) \right)^2.
    \label{eq11}
\end{equation}
In our implementation, we compute the diagonal of FIM in an online manner by accumulating the back-propagated gradients at each batch during training, leading to negligible additional computation cost.

Building on Proposition \ref{prop1}, we propose a novel method for CIL named Increment Vector Transformation (IVT). 
By replacing the Hessian in Equation \ref{eq10} with Equation \ref{eq11}, we formally define the iterative update rule of IVT as follows,
\begin{equation}
    \hat{\theta}_{t} := \hat{\theta}_{t-1} + (\bar{F}_{t-1}+\bar{F}_{t})^{-1}\bar{F}_{t}(\theta_t - \hat{\theta}_{t-1}). \label{eq12}
\end{equation}
Here, $\bar{F}_{t}= \sum_{i=1}^{t}F_i$ represents the cumulative diagonal of the FIM up to task $t$. $\hat{\theta}_{t}$ denotes the IVT model.

The transformation in Equation~\ref{eq12} is applied periodically—every few epochs—instead of post-hoc. This is essential because IVT relies on a local second-order approximation around the previous task optimum $\theta_{t-1}^*$ to guide the increment vector. If the model drifts too far due to SGD, higher-order curvature invalidates this approximation, and the transformed vector may fail to preserve low-loss connectivity. Periodic updates keep the model within the local region where the approximation remains valid, ensuring both efficiency and theoretical soundness.
Thus, IVT functions as an online trajectory controller rather than a post-hoc fix. It involves only simple matrix operations on parameter vectors, incurs negligible overhead, and is easy to implement—requiring just a few lines of PyTorch code. The method integrates seamlessly into existing CIL frameworks, and pseudocode is provided in Algorithm~\ref{algo}.

\begin{algorithm}[t]
\caption{Increment Vector Transformation (IVT)}
\label{algo}
\begin{algorithmic}[1]
    \STATE Train $\hat{\theta}_1 = \theta_1$ on task $\mathcal{T}_1$
    \STATE Compute $F_1 = \mathbb{E}_{(x, y) \in \mathcal{T}_1} \left[\nabla \mathcal{L}_1(x, y)\right]^2$
    \FOR{incremental task $\mathcal{T}_t \in \{\mathcal{T}_2, \mathcal{T}_3, \cdots \}$}
        \STATE Initialize $\theta_t^{(1)} = \hat{\theta}_{t-1}$
        \FOR{epoch $m = 1, 2, \dots, M$}
            \STATE Initialize $F_t = \mathbf{0}$
            \FOR{mini-batch $\mathcal{B}_i$ in training data}
                \STATE Compute $g_i = \mathbb{E}_{(x, y) \in \mathcal{B}_i} \left[ \nabla \mathcal{L}_t(x, y) \right]$
                \STATE Update $\theta_t^{(m)} \gets \texttt{CILMethod}(\theta_t^{(m)}, g_i)$
            \ENDFOR
            \STATE Compute $F_t = \mathbb{E}_i(g_i^2)$ 
            
            \COMMENT{IVT is applied every $I$ epochs}
            
            \IF{$m \bmod I = 0$}
                \STATE $\theta_t^{(m)} \gets \hat{\theta}_{t-1} + (\bar{F}_{t-1} + \bar{F}_t)^{-1} \bar{F}_t (\theta_t^{(m)} - \hat{\theta}_{t-1})$
            \ENDIF
        \ENDFOR
        \STATE Update Fisher: $\bar{F}_t =\bar{F}_{t-1} + F_t$
        \STATE Output: $\hat{\theta}_t =\theta_t^{(M)}$
    \ENDFOR
\end{algorithmic}
\end{algorithm}

\section{Experiment}

\paragraph{Datasets}
We conduct extensive experiments on four representative datasets: CIFAR-100 \cite{krizhevsky2009Learning}, FGVCAircraft \cite{maji2013fine}, ImageNet-Subset \cite{Deng2009ImageNetAL}, and ImageNet-Full \cite{Deng2009ImageNetAL}.
CIFAR-100 is a widely used benchmark for continual learning, consisting of 60,000 $32\times32$ images across 100 object categories with significant inter-class similarity. FGVCAircraft is a fine-grained classification dataset containing 6,667 images from 100 aircraft variants, which presents greater challenges for knowledge retention due to subtle visual differences between classes.
ImageNet-Subset and ImageNet-Full are derived from the large-scale ImageNet ILSVRC dataset, with the subset containing 100 randomly sampled classes and the full version covering all 1,000 classes. These datasets offer increasing complexity and diversity, enabling us to evaluate scalability and generalization across extensive tasks.
For methods trained from scratch, the initial task comprises half of the classes to allow for meaningful model initialization, while the remaining classes are evenly distributed across subsequent incremental tasks. In contrast, for pre-trained continual learning methods, all classes are uniformly distributed across tasks. To ensure reproducibility and eliminate bias from class sequencing, the class order is randomized using a fixed seed (1993).

\paragraph{Evaluation Metrics}
Following standard evaluation protocols, we report average accuracy: $AA = \frac{1}{T} \sum_{t=1}^{T} a_t$, and last accuracy: $LA = a_T$, where $a_t$ denotes accuracy over all seen classes after task $t$. Forgetting is measured as \cite{chaudhry2018riemannian}: $FM = \frac{1}{T-1} \sum_{i=1}^{T-1}\max_{t \in \{i, T-1\}} (a_{t, i} - a_{T, i})$, with $a_{t,i}$ denoting task $i$ accuracy after training on task $t$. 
The Average Improvement across adapted methods is reported:
$Avg.\  Imp. = \frac{1}{N} \sum_{t=1}^{N} M'-M$ where $M$ and $M'$ are the metric values before and after applying IVT. 

\paragraph{Comparison Methods}
IVT is orthogonal to existing CIL approaches, serving as a plug-in to enhance their performance. To verify the general effectiveness of IVT, we integrate it with exemplar-based methods: PODNet \cite{Douillard2020PODNetPO} and MRFA \cite{zheng2024multi},  non-exemplar methods: PASS \cite{zhu2021prototype} and FCS \cite{li2024fcs}, and pre-trained CIL methods: FLYP \cite{goyal2023finetune}, SLCA \cite{zhang2023slca}. 
Additionally, we report results for iCaRL \cite{Rebuffi2016iCaRLIC}, BiC \cite{wu2019large}, LUCIR \cite{Hou2019LearningAU}, Mnemonics \cite{liu2020mnemonics}, GeoDL \cite{simon2021learning}, EOPC \cite{wen2023optimizing}, SSRE \cite{zhu2022self}, SPU \cite{zhang2024overcoming}, RAPF \cite{huang2024class}, CoFiMA \cite{marouf2024weightedensemblemodelsstrong} and MagMax \cite{marczak2024magmax} as baseline comparisons. 

\paragraph{Implementation Details}
For methods trained from scratch, we use ResNet-32 \cite{he2016deep} (stride 8) for CIFAR-100 and ResNet-18 \cite{he2016deep} (stride 32) for ImageNet. Exemplar-based methods are trained with SGD (initial learning rate 0.1) for 160 epochs on CIFAR-100 and 90 epochs on ImageNet. The memory size $|\mathcal{M}|$ is set to 20 samples per class unless stated otherwise.
For non-exemplar methods, we use Adam (learning rate 0.001) for 100 epochs, while pre-trained methods use CLIP-ViT/B-16 \cite{radford2021learning} with Adam (learning rate 5e-6)  for 10 epochs. 
For scratch-trained methods, the batch size is 128 and the IVT interval is 10 epochs. For pre-trained methods, the batch size is 64 and the IVT interval is 3 epochs.

\begin{table*}[t]
\centering
\caption{Comparative results (\%) on CIFAR-100 with different numbers of incremental tasks. The results are averaged over 3 random runs, with both the mean and standard deviation reported. } 
\setlength{\tabcolsep}{4.2pt}
\resizebox{1\textwidth}{!}{%
\begin{tabular}{
lc
ccc|
ccc|
ccc
}
\toprule[1pt]
\multirow{2}{*}{Method}  &Exemplar & \multicolumn{3}{c|}{5 Tasks} & \multicolumn{3}{c|}{10 Tasks} & \multicolumn{3}{c}{25 Tasks} \\ 
& -Free &  {$AA \uparrow$}   &       {$LA \uparrow$}  &  {$FM \downarrow$} &  {$AA \uparrow$}   &       {$LA \uparrow$}  &  {$FM \downarrow$} &  {$AA \uparrow$}   &       {$LA \uparrow$}  &  {$FM \downarrow$} \\
\midrule
iCaRL \cite{Rebuffi2016iCaRLIC} & \xmark & 57.82 & 47.46 & 19.35 & 53.13 & 44.55 & 20.54 & 48.64 & 40.32 & 23.87 \\
LUCIR \cite{Hou2019LearningAU} & \xmark & 63.32 & 55.07 & 19.30 & 61.10 & 52.46 & 22.02 & 58.39 & 48.57 & 27.16 \\
EOPC \cite{wen2023optimizing} & \xmark & 65.06 & 55.69 & 8.93 & 63.64 & 54.05 & 8.24 & 61.35 & 51.24 & 11.94 \\
SSRE \cite{zhu2022self} & \cmark  & 53.01 & 44.23 & 15.46 & 51.17 & 42.68 & 14.47 & 50.40 & 41.80 & 12.44\\
CoFiMA \cite{marouf2024weightedensemblemodelsstrong} & \cmark & 36.51 & 7.72 & 42.16 & 26.34& 2.37 & 35.97 & 16.00 & 1.81 & 21.58\\ 
MagMax \cite{marczak2024magmax} & \cmark  & 29.11 & 3.70 &39.32 & 19.69 & 4.10 &46.06 & 16.63 & 5.21 & 52.99\\
\midrule
PODNet \cite{Douillard2020PODNetPO} & \xmark & 64.00\textsubscript{ (0.54)} & 54.47\textsubscript{ (0.88)} & 17.72\textsubscript{ (0.27)} & 62.47\textsubscript{ (0.51)} & 52.89\textsubscript{ (0.80)} & 21.57\textsubscript{ (0.38)} & 59.82\textsubscript{ (0.84)} & 50.71\textsubscript{ (0.96)} & 25.90\textsubscript{ (0.89)}\\ 
 \rowcolor{ cyan!5} \hspace{3pt} \textit{w/} IVT (Ours)  & \xmark & 65.36\textsubscript{ (0.24)} &  56.62\textsubscript{ (0.47)} & 11.68\textsubscript{ (0.47)} & 63.45\textsubscript{ (0.72)} & 55.41\textsubscript{ (0.72)} & 12.87\textsubscript{ (0.47)} & 61.74\textsubscript{ (0.98)} & 53.43\textsubscript{ (1.15)} & 15.84\textsubscript{ (0.76)}\\ 
MRFA \cite{zheng2024multi} & \xmark & 61.85\textsubscript{ (0.14)} & 51.28\textsubscript{ (0.13)}  & 21.39\textsubscript{ (0.24)}& 58.63\textsubscript{ (0.24)} & 48.46\textsubscript{ (0.29)}  & 23.29\textsubscript{ (0.94)}&55.10\textsubscript{ (0.22)} & 44.59\textsubscript{ (0.42)}  & 25.18\textsubscript{ (0.89)}  \\ 
 \rowcolor{ cyan!5}\hspace{3pt} \textit{w/} IVT (Ours) & \xmark& 63.50\textsubscript{ (0.07)} & 53.79\textsubscript{ (0.07)}  & 8.09\textsubscript{ (0.27)}& 61.86\textsubscript{ (0.10)} & 51.51\textsubscript{ (0.07)}  & 7.25\textsubscript{ (0.32)} & 60.45\textsubscript{ (0.03)} & 48.74\textsubscript{ (0.03)}  & 7.70\textsubscript{ (0.04)} \\ 
PASS \cite{zhu2021prototype}& \cmark  & 58.60\textsubscript{ (0.11)} & 49.05\textsubscript{ (0.21)} &  8.36\textsubscript{ (0.08)} & 56.55\textsubscript{ (0.65)} & 44.95\textsubscript{ (1.43)} &  11.75\textsubscript{ (1.26)} & 54.42\textsubscript{ (0.34)} & 41.24\textsubscript{ (0.53)} &  13.29\textsubscript{ (0.46)}\\
 \rowcolor{ cyan!5} \hspace{3pt} \textit{w/} IVT (Ours)& \cmark & 59.29\textsubscript{ (0.10)} & 49.75\textsubscript{ (0.13)} & 8.14\textsubscript{ (0.35)} & 58.04\textsubscript{ (0.63)} & 46.78\textsubscript{ (1.20)}& 11.29\textsubscript{ (1.01)}  & 57.81\textsubscript{ (0.13)} & 46.36\textsubscript{ (0.47)} & 10.75\textsubscript{ (0.14)}\\
FCS \cite{li2024fcs} & \cmark & 60.04\textsubscript{ (0.14)} & 50.87\textsubscript{ (0.22)}& 7.61\textsubscript{ (0.09)}& 59.64\textsubscript{ (0.06)}& 49.71\textsubscript{ (0.32)}&  8.92\textsubscript{ (0.15)}& 58.40\textsubscript{ (0.06)}& 46.90\textsubscript{ (0.13)}& 10.98\textsubscript{ (0.27)} \\
 \rowcolor{ cyan!5} \hspace{3pt}\textit{w/} IVT (Ours)& \cmark & 60.96\textsubscript{ (0.12)} & 52.07\textsubscript{ (0.05)}&  6.80\textsubscript{ (0.42)}& 60.99\textsubscript{ (0.14)}& 51.63\textsubscript{ (0.15)}&  7.79\textsubscript{ (0.26)} & 60.18\textsubscript{ (0.08)}& 49.72\textsubscript{ (0.21)}&  7.00\textsubscript{ (0.47)} \\
 \addlinespace[1pt]
\hdashline
\addlinespace[2pt]
  \rowcolor{cyan!10} \textit{Avg. Imp.}& & \textit{+1.16} & \textit{+1.64}& \textit{-5.59
} & \textit{+1.76}  & \textit{+2.33} & \textit{-6.58}& \textit{+3.11} & \textit{+3.70} & \textit{-8.52} \\
\bottomrule[1pt]
\end{tabular}}
\label{tab:cifar}
\end{table*}

\begin{table*}[t]
\centering
\caption{Comparative results (\%) on ImageNet-Subset and ImageNet-Full with different numbers of incremental tasks. Results marked with $^\dag$ are referenced from \cite{simon2021learning}.} 
\setlength{\tabcolsep}{6pt}
\resizebox{1\textwidth}{!}{%
\begin{tabular}{
  lcccc|ccc|ccc|ccc
}
\toprule[1pt]
\multirow{3}{*}{Method}  &   & \multicolumn{9}{c|}{ImageNet-Subset} & \multicolumn{3}{c}{ImageNet-Full} \\ 
&\multirow{2}{*}{Exemplar}  & \multicolumn{3}{c|}{5 Tasks} & \multicolumn{3}{c|}{10 Tasks} & \multicolumn{3}{c|}{25 Tasks}  & \multicolumn{3}{c}{10 Tasks} \\ \cmidrule[0.5pt](lr){3-14}
&  -Free &   {$AA \uparrow$}   &   {$LA \uparrow$}  &  {$FM \downarrow$} &  {$AA \uparrow$}   &   {$LA \uparrow$}  &  {$FM \downarrow$} &  {$AA \uparrow$}   &   {$LA \uparrow$}  &  {$FM \downarrow$} &  {$AA \uparrow$}   &   {$LA \uparrow$}  &  {$FM \downarrow$}\\
\midrule
iCaRL$^\dag $ \cite{Rebuffi2016iCaRLIC} & \xmark & 65.44 & 53.14 & {--} & 59.88 & 49.20 & {--} & 52.97 & {--} & {--} & 46.89 & 38.69 & {--} \\
BiC$^\dag $  \cite{wu2019large}& \xmark  & 70.07 & 60.34 & {--} & 64.96 & 56.18 & {--} & 57.73 & {--} & {--} & 58.72 & 51.23 & {--} \\
LUCIR$^\dag $ \cite{Hou2019LearningAU} & \xmark & 70.84 & 60.39 & {--}  & 68.32 & 57.95 & {--}  & 61.44 & {--} & {--}  & 61.63 & 52.73 & {--}  \\
Mnemonics$^\dag $ \cite{liu2020mnemonics} & \xmark & 72.58 & 64.58 & {--} & 71.37 & 62.52 & {--} & 69.74 & {--} & {--} & 63.01 & 55.45 & {--} \\
GeoDL$^\dag $ \cite{simon2021learning}& \xmark  & 73.87 &  67.37 & {--} & 73.55 & 65.57 & {--} & 71.72 & {--} & {--} & 64.46 & 56.75  & {--} \\

SSRE \cite{zhu2022self}  & \cmark &  65.59 & 55.08 & 26.16 & 63.25 & 53.38 & 29.06 & 62.02 & 53.00 & 32.06  & -- & -- & --\\
\midrule
PODNet \cite{Douillard2020PODNetPO} & \xmark & 73.53 & 62.96 & 17.26  & 69.36 & 59.60 & 21.59 & 60.40 & 47.76 & 32.93 & 64.10 & 55.57 & 14.09\\ 
 \rowcolor{ cyan!5}\hspace{3pt} \textit{w/} IVT (Ours)& \xmark & 74.48 & 65.48 & 11.11 & 71.63 & 61.86 & 15.01 & 65.78 & 54.66 & 20.38 &  65.07 & 56.95 &  13.00 \\ 
MRFA \cite{zheng2024multi}& \xmark& 65.68 & 51.84
 & 35.01 & 59.93 & 46.08 &39.07 & 54.46 & 41.98 & 41.35 & 46.65 & 36.99 & 18.70 \\ 
 \rowcolor{ cyan!5}\hspace{3pt} \textit{w/} IVT (Ours) & \xmark& 75.10 & 65.14 & 13.16
 & 73.02 & 61.60 &12.38 & 70.00 & 56.66 & 12.15 & 59.48 & 48.34 & 10.08\\ 
PASS \cite{zhu2021prototype} & \cmark &  68.10 & 54.20 & 24.72 & 65.00 & 51.92 & 29.31  & 57.36 & 39.16 & 41.99 & 48.74 & 40.42 & 19.05\\
 \rowcolor{ cyan!5}\hspace{3pt} \textit{w/} IVT (Ours) & \cmark& 70.55 & 57.48 & 20.38 & 68.49 & 56.72 & 20.54 & 60.64 & 40.64 & 34.96 & 50.22 & 42.09 & 16.37 \\
FCS  \cite{li2024fcs} & \cmark &  74.06 & 63.82 & 17.78  & 73.55 & 63.74 & 20.72 & 70.50 & 56.56 & 26.00 & -- & -- & --\\
\rowcolor{ cyan!5}\hspace{3pt} \textit{w/} IVT (Ours) & \cmark & 76.98 & 66.52 & 14.54 & 76.15 & 66.08 & 16.18 & 72.74 & 57.90 & 24.65& -- & -- & -- \\
 \addlinespace[1pt]
\hdashline
\addlinespace[2pt]
  \rowcolor{cyan!10} \textit{Avg. Imp}. & & \textit{+3.93} & \textit{+5.45}& \textit{-8.90} & \textit{+5.36}  & \textit{+6.23} & \textit{-11.65}& \textit{+6.61} & \textit{+6.10} & \textit{-12.53} & \textit{+5.09} & \textit{+4.80} & \textit{-4.13} \\
\bottomrule[1pt]
\end{tabular}}
\label{tab:imagenet}
\end{table*}

\subsection{Comparative Results}

\paragraph{Adaptation with Scratch-trained CIL}
We integrate IVT with four representative scratch-trained CIL methods—PODNet, MRFA, PASS, and FCS—and report the results in Table\ref{tab:cifar} and Table\ref{tab:imagenet}. Across all methods and settings, IVT consistently improves average accuracy, last task accuracy, and forgetting measure. These improvements are particularly pronounced in more challenging configurations with more incremental tasks (e.g., 25-task settings), indicating IVT's effectiveness in mitigating forgetting over long horizons.
Notably, we also include CoFiMA and MagMax—two recent post-hoc model merging methods—as baselines. As shown in Table~\ref{tab:cifar}, both methods perform poorly under the scratch-trained setting, with accuracy dropping below 30\% in many configurations. This aligns with findings in prior work such as \cite{neyshabur2020being}, which showed that model solutions obtained from disjoint training trajectories often lie in disconnected regions of the loss landscape, resulting in performance barriers along linear paths. As a result, post-hoc approaches that rely on parameter interpolation between such solutions (\textit{e.g.}, CoFiMA, MagMax) are inherently limited in this setting.
These results highlight the advantage of our IVT approach, which explicitly adjusts the model's optimization direction \emph{during training} to remain aligned with low-loss regions. By actively steering updates toward LMC-consistent directions, IVT avoids the connectivity collapse observed in scratch-trained models after multiple optimization steps, leading to more stable task integration and reduced forgetting.

Moreover, IVT also yields consistent improvements for exemplar-free baselines such as PASS and FCS. These results confirm that IVT is not only compatible with exemplar-based strategies but also highly effective in exemplar-free scenarios, where preserving prior knowledge is more challenging due to the absence of rehearsal. This underscores the generality and robustness of IVT across different CIL paradigms.

\paragraph{Adaptation with Pre-trained CIL} 
As shown in Table~\ref{tab:pre-train}, we integrate IVT with several pre-trained CIL methods on CIFAR-100 and FGVCAircraft using CLIP-ViT/B-16 as the backbone. The results confirm the effectiveness of IVT in the pre-trained setting. While pre-training typically leads to a flatter loss landscape that helps mitigate forgetting, IVT further steers the optimization trajectory toward more robust region that may not have been sufficiently leveraged during pre-training. Notably, when applied to simple baselines such as FLYP and SLCA, IVT brings substantial gains across all metrics and datasets, consistently outperforming recent state-of-the-art methods such as RAPF and MagMax. For example, on CIFAR-100 with 10 tasks, IVT improves FLYP’s last accuracy by over 23\% and reduces forgetting by more than 26\%. These consistent improvements across both coarse-grained and fine-grained benchmarks validate that IVT is not only effective in scratch-trained paradigms but also readily generalizes to strong pre-trained continual learners.

\begin{table*}[t]
\centering 
\caption{Comparative results (\%) on CIFAR-100 and FGVCAircraft using pre-trained CLIP-ViT/B-16.} 
\setlength{\tabcolsep}{5pt}
\resizebox{1\linewidth}{!}{
\begin{tabular}{lcccc|ccc|ccc|ccc}
\toprule[1pt]
\multirow{3}{*}{Method}  &  & \multicolumn{6}{c|}{CIFAR-100} & \multicolumn{6}{c}{FGVAircraft} \\ 
& \multirow{2}{*}{Exemplar} & \multicolumn{3}{c|}{5 Tasks} & \multicolumn{3}{c|}{10 Tasks} & \multicolumn{3}{c|}{5 Tasks}  & \multicolumn{3}{c}{10 Tasks} \\ \cmidrule[0.5pt](lr){3-14}
&  -Free &   {$AA \uparrow$}   &   {$LA \uparrow$}  &  {$FM \downarrow$} &  {$AA \uparrow$}   &   {$LA \uparrow$}  &  {$FM \downarrow$} &  {$AA \uparrow$}   &   {$LA \uparrow$}  &  {$FM \downarrow$} &  {$AA \uparrow$}   &   {$LA \uparrow$}  &  {$FM \downarrow$}\\
\midrule
Zero-Shot CLIP \cite{radford2021learning} & - & - & 68.25 & - & - & 68.25 & - & - & 24.45 & - & - & 24.45 & -\\

 SPU \cite{zhang2024overcoming} & \cmark & 84.24& 75.91 & -0.73 &  81.10& 67.45 & 16.40 & 49.99 & 36.24 & 11.19 & 46.87 & 28.29 & 20.85\\
RAPF  \cite{huang2024class} & \cmark & 85.78 & 79.77 & 5.73 & 86.05 &78.64 & 9.59& 50.25 & 33.96 &27.06 & 44.14 & 27.66 & 43.96 \\
CoFiMA  \cite{marouf2024weightedensemblemodelsstrong} & \cmark & 82.25 & 75.31& -12.38 & 78.49 & 72.09 & -5.21  & 58.01 & 40.44 & 11.98 & 50.56 & 32.04 & 19.57 \\
MagMax \cite{marczak2024magmax} & \cmark & 80.42 &  70.74 &  -10.85 & 78.53 & 68.86 & -3.87   & 54.76 & 36.63 & 3.18 & 46.57 & 29.55& 8.75 \\
\midrule
FLYP \cite{zhang2023slca} & \cmark &  77.80 & 64.64 & 10.52 &  69.86 & 52.45 & 29.43   & 57.61 & 39.87 & 20.64 & 45.02 & 23.40 & 39.59 \\
 \rowcolor{ cyan!5}\hspace{3pt} \textit{w/} IVT (Ours) & \cmark  &  87.72 & 80.13 & -9.74 &  85.98 & 75.81 & 3.12  & 58.81 & 40.95 & 8.92 & 52.20 & 33.36 & 13.39 \\
 SLCA \cite{zhang2023slca}& \cmark &79.84  &68.55& 8.21 & 71.39  & 55.40  & 31.97 & 54.93 & 35.52 & 22.05 & 42.40 & 25.29 & 37.02\\
 \rowcolor{ cyan!5}\hspace{3pt} \textit{w/} IVT (Ours) & \cmark & 88.12 & 81.37 & -7.68 &  86.32 & 77.35 & 5.98 & 61.29 & 45.75 & 8.79 & 55.06 & 40.29 & 18.39\\
  \addlinespace[1pt]
\hdashline
\addlinespace[2pt]
\rowcolor{cyan!10} \textit{Avg. Imp}. & - 
& \textit{+9.10} & \textit{+14.16} & \textit{-18.08} 
& \textit{+15.52} & \textit{+22.65} & \textit{-26.15} 
& \textit{+3.78} & \textit{+5.66} & \textit{-12.49} 
& \textit{+9.92} & \textit{+12.48} & \textit{-22.42} \\
\bottomrule[1pt]
\end{tabular}} 
\label{tab:pre-train}
\end{table*}

\begin{figure*}[t]
\centering
\subfloat{\includegraphics[width=.32\linewidth]{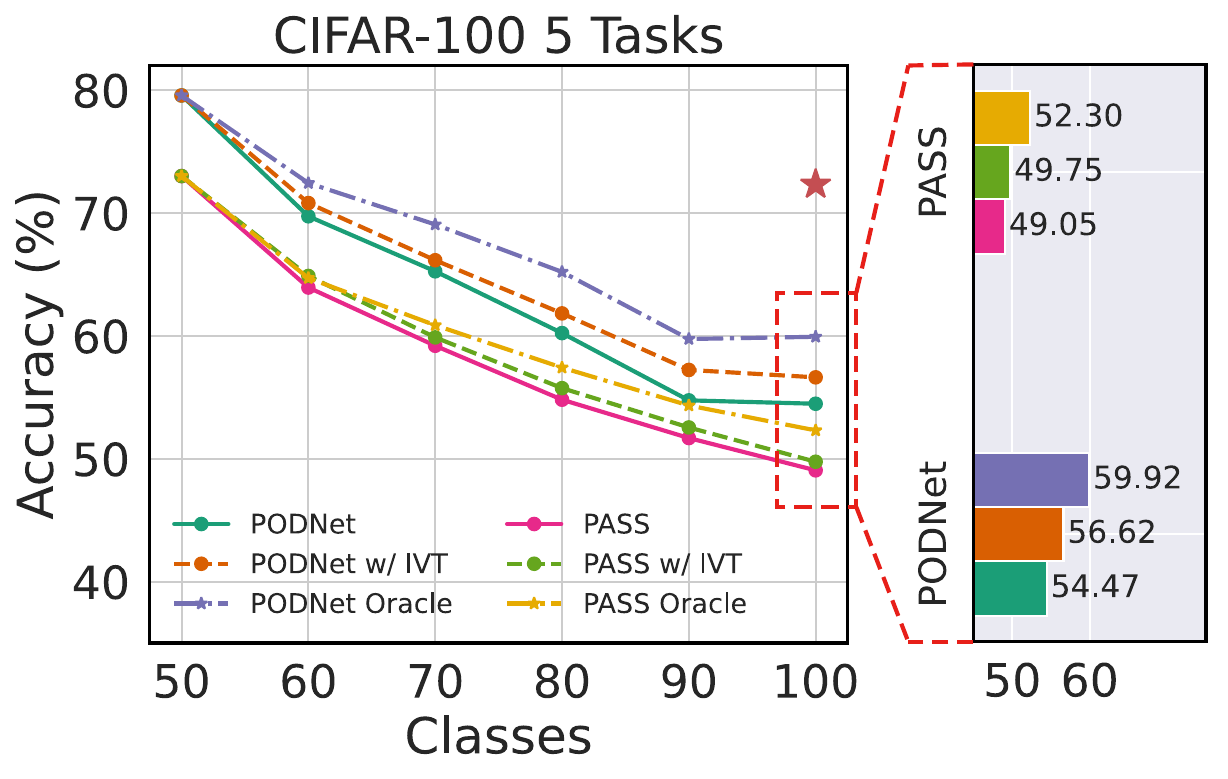}} \hfill
\subfloat{\includegraphics[width=.32\linewidth]{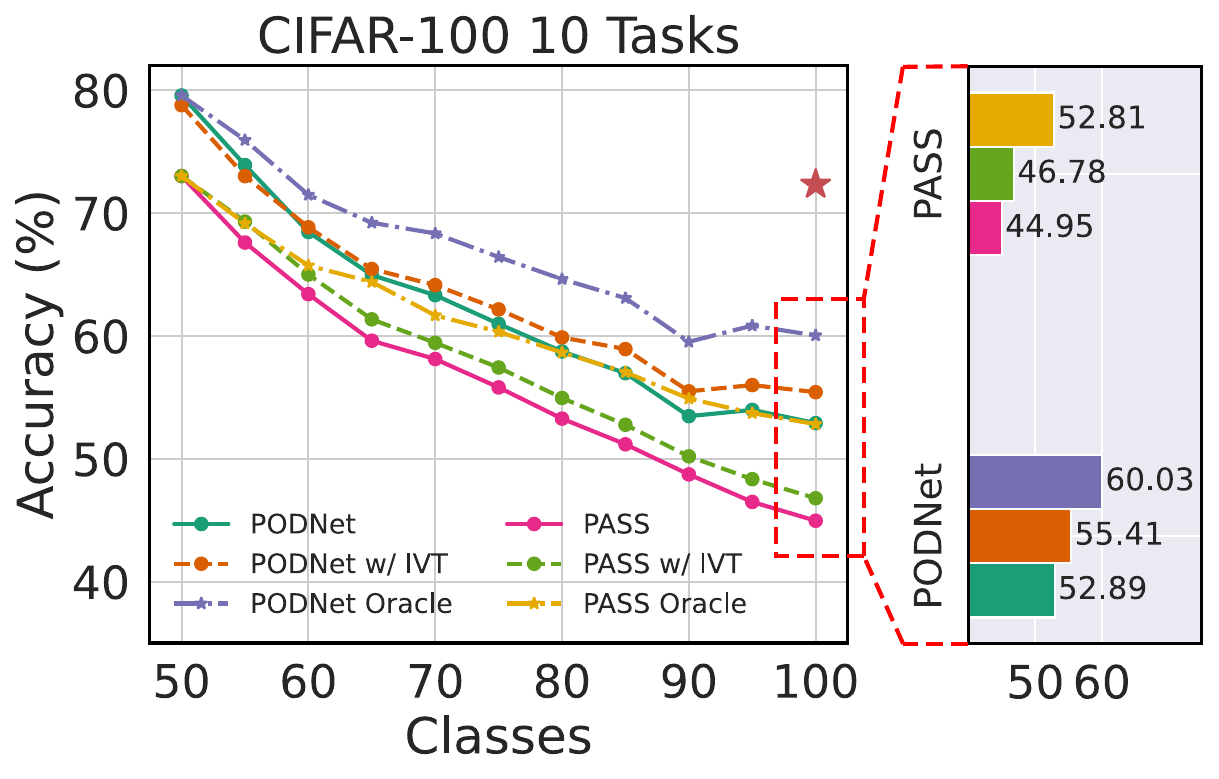}} \hfill
\subfloat{\includegraphics[width=.32\linewidth]{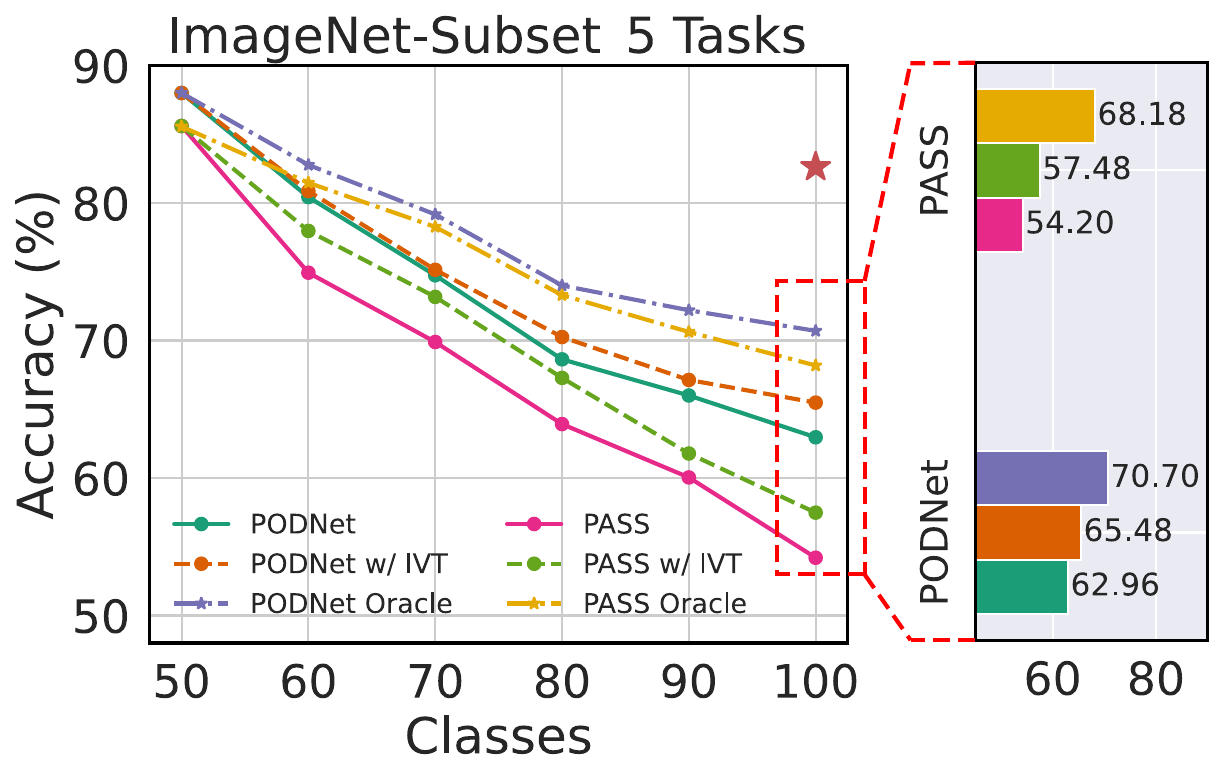}} 
\caption{Comparison results with oracles of PODNet and PASS. The bar chart on the right shows the last accuracy.
The oracle is trained incrementally with full data replay, while the multi-task optimum (red star) represents training on the entire dataset at once.}
\label{fig:oracle}
\end{figure*}

\paragraph{Comparison with CIL Oracles}
We further compare IVT with CIL oracles, obtained via incremental training with full access to historical data, as shown in Figure~\ref{fig:oracle}. Notably, the CIL oracle differs from the multi-task optimum (MTL) in two key aspects:
(1) The MTL optimum is derived by jointly training on all tasks at once with perfectly balanced and temporally aligned data—an idealized setting that violates the sequential nature of continual learning. In contrast, the CIL oracle starts from the previous model and incrementally trains on replayed data, preserving the structure and constraints of the underlying CIL method.
(2) The CIL oracle incorporates anti-forgetting mechanisms (\textit{e.g.}, regularization or distillation) that promote stability but may limit plasticity, while the MTL model has no such constraints. As a result, the CIL oracle often underperforms MTL in last accuracy, but provides a more realistic and attainable upper bound for continual learners operating under practical assumptions.
Crucially, our results show that IVT effectively narrows the gap between strong CIL baselines and their corresponding oracles, demonstrating improved retention and forward transfer. This highlights IVT's potential to bring practical CIL systems closer to their achievable limits, without relying on idealized joint training.

\subsection{Analytical Experiments}

\paragraph{Linear Mode Connectivity along the Linear Path}
Similar to Section \ref{sec2}, we analyze the LMC of the IVT model. As shown in Figure \ref{lmc_2}, the IVT model exhibits behavior closely aligned with the CIL oracle along the linear path: as $\lambda$ increases, the IVT model experiences only a slight drop in accuracy, while maintaining a comparable distance from the old model as the oracle. This observation suggests that both the IVT model and the oracle remain within the same loss basin as the previous model, thereby preserving performance on previously learned tasks.
Furthermore, Figure \ref{blc_2} demonstrates that the IVT model achieves a stability-plasticity trade-off similar to the oracle. Compared to Figure \ref{blc_1}, IVT substantially mitigates the conflict between retaining old knowledge and acquiring new information, enabling more balanced continual adaptation. This underscores the robustness of IVT in navigating the optimization landscape: it not only avoids catastrophic forgetting but also promotes smoother transitions between tasks, a property often absent in conventional CIL baselines.
Collectively, these findings confirm that IVT preserves the geometric structure of optimal solutions across tasks, providing both empirical and theoretical support for its capacity to approximate the oracle in terms of loss connectivity and transfer dynamics.

\begin{figure*}[t]
\centering
\subfloat{\includegraphics[width=.24\linewidth]{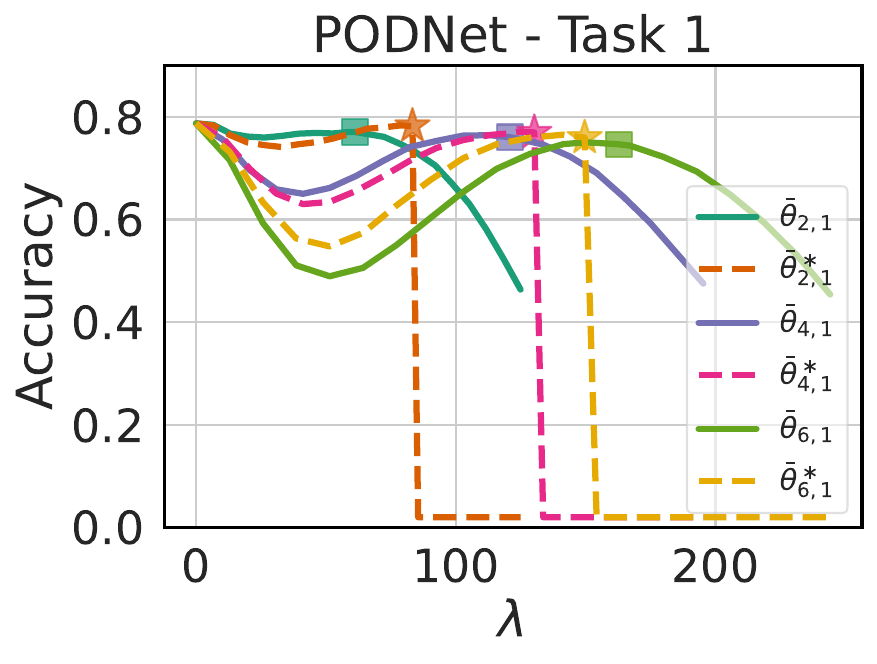}} \hfill
\subfloat{\includegraphics[width=.24\linewidth]{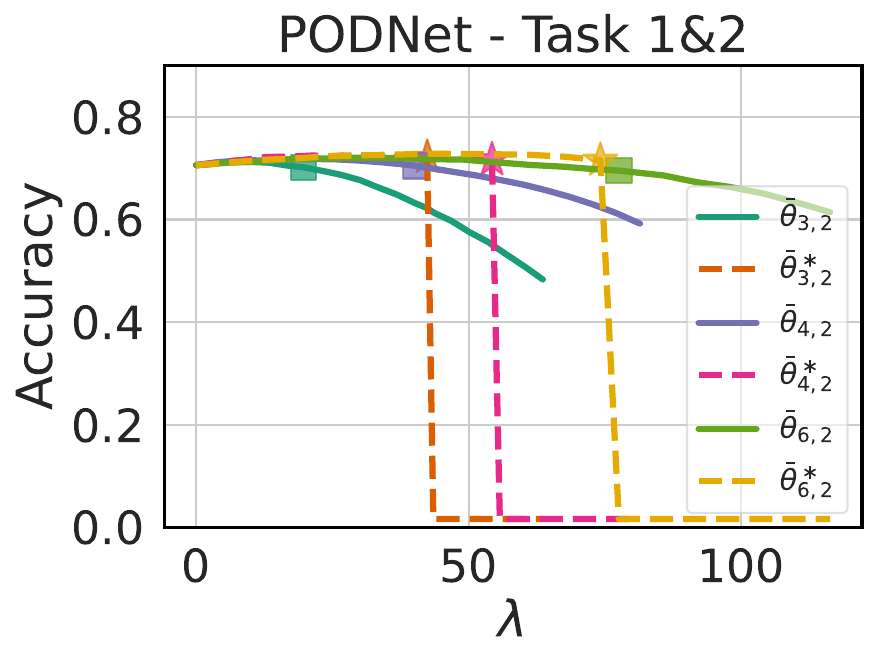}} \hfill
\subfloat{\includegraphics[width=.24\linewidth]{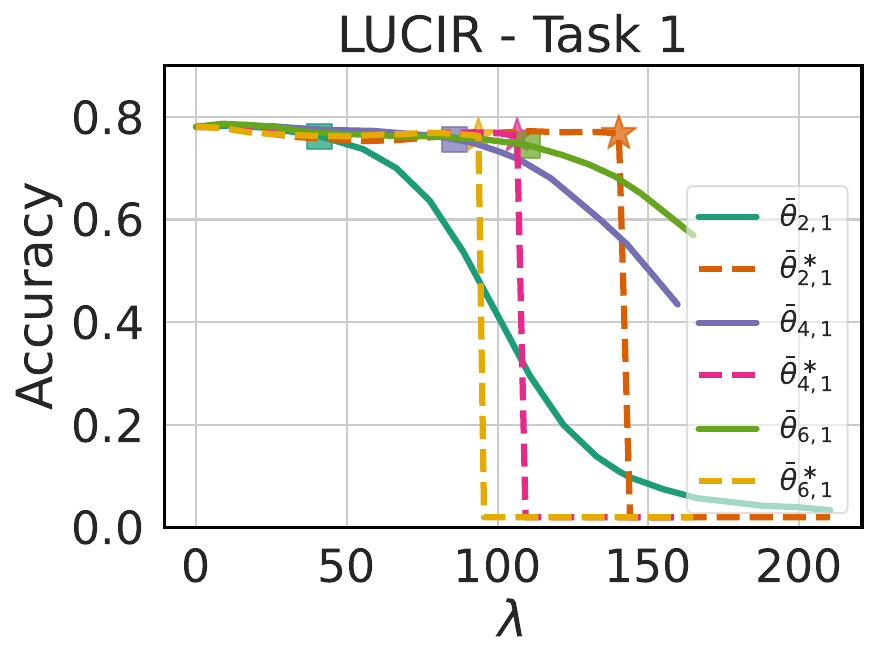}} \hfill
\subfloat{\includegraphics[width=.24\linewidth]{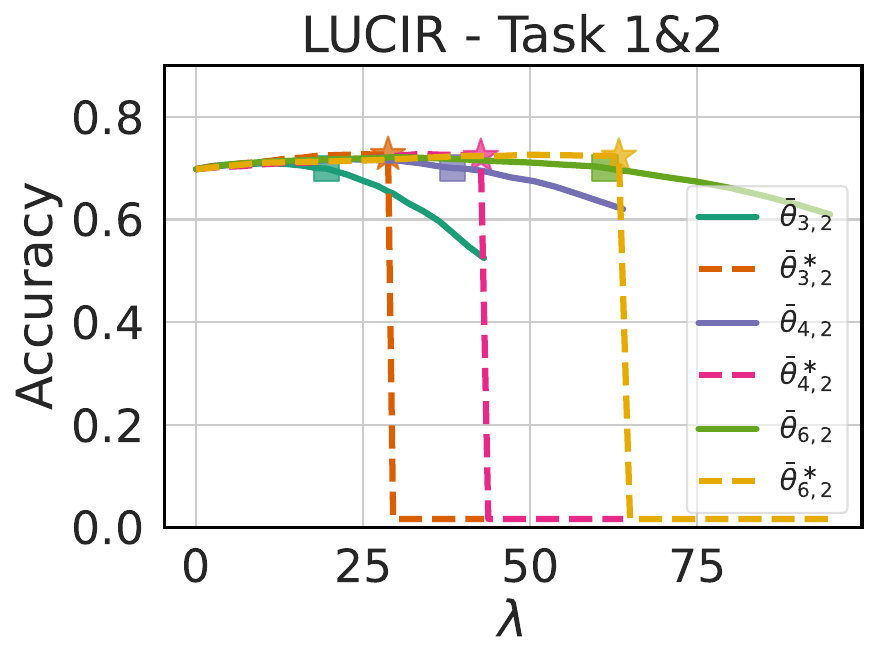}}  
\caption{Evaluating accuracy consistency along the linear path on CIFAR-100 for increments of 5 tasks.  The star and square denote the CIL oracle $\theta_{t}^*=\bar{\theta}^*(\hat{\lambda}^*)$ and the IVT model $\theta_{t}=\bar{\theta}(\hat{\lambda})$.}  
\label{lmc_2}
\end{figure*}

\begin{figure*}[t]
\centering
\hspace{10pt}\subfloat{\includegraphics[width=.3\linewidth]{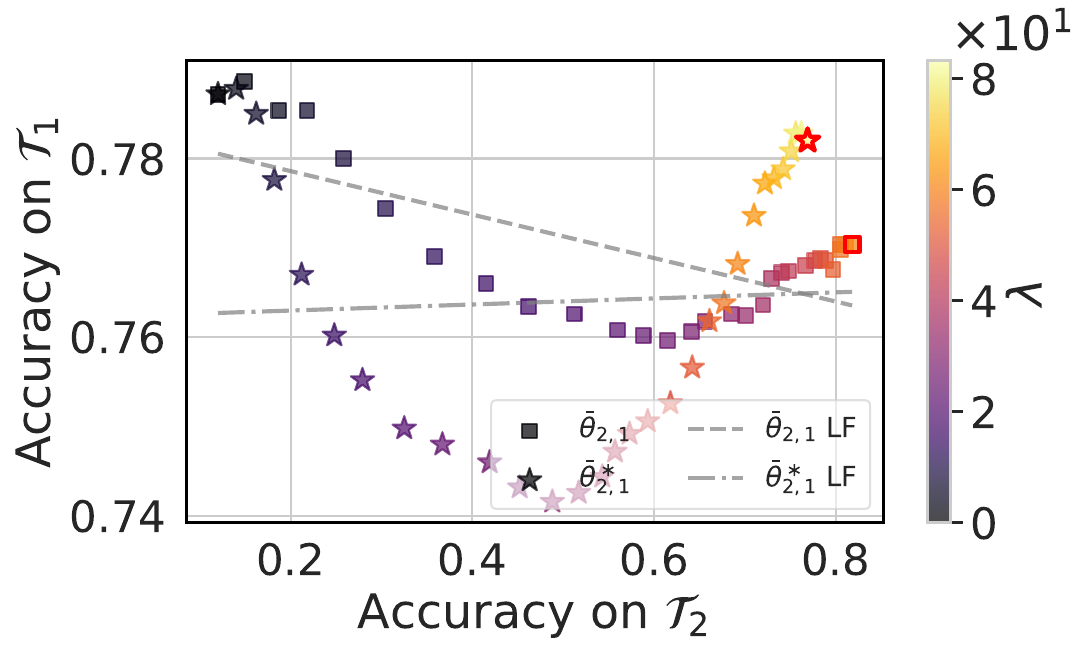}}\hfill
\subfloat{\includegraphics[width=.3\linewidth]{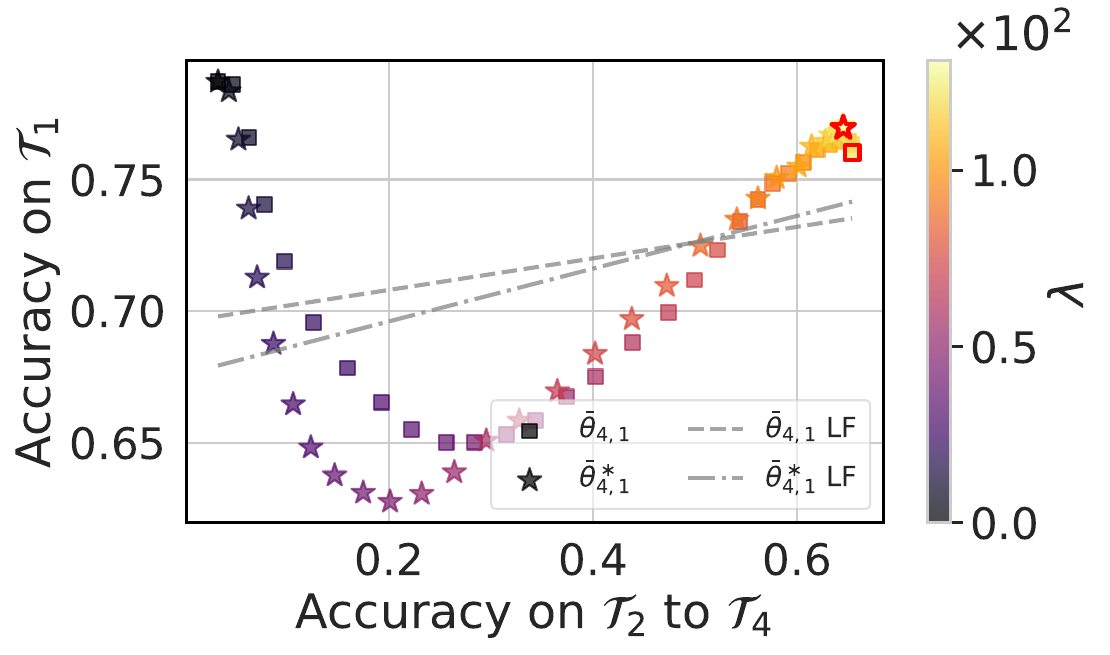}}\hfill
\subfloat{\includegraphics[width=.3\linewidth]{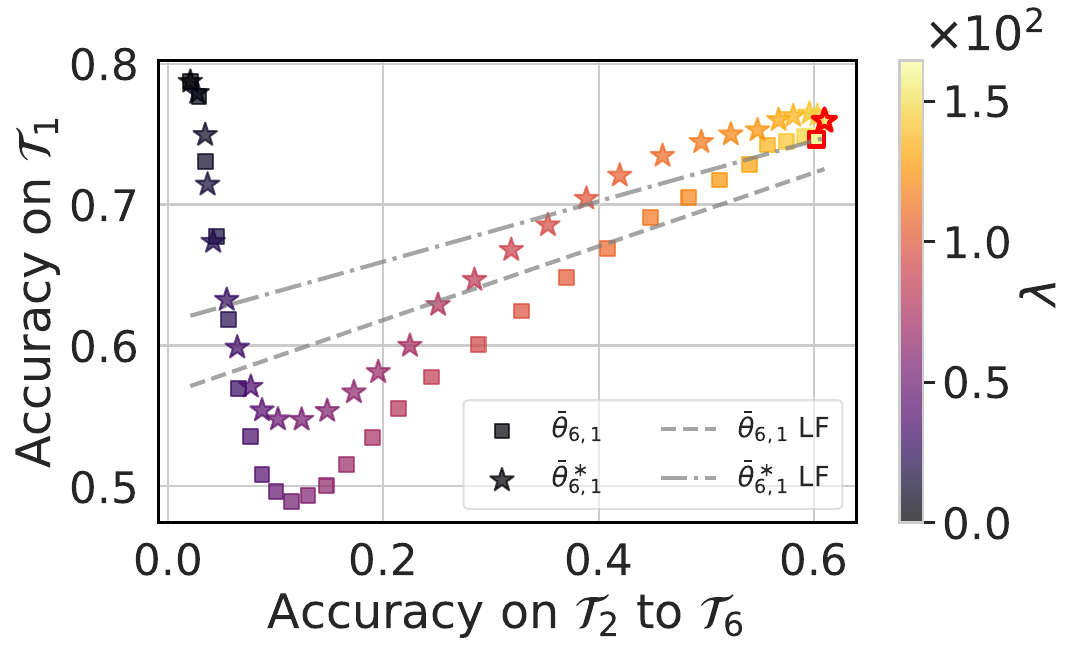}} \hspace{10pt}
\caption{Evaluating stability-plasticity trade-off achieved by PODNet along the linear path on CIFAR-100 for increments of 5 tasks. The red-edged star and square denote the CIL oracle $\bar{\theta}^*(\hat{\lambda}^*)$ and the IVT model $\bar{\theta}(\hat{\lambda})$, respectively. } 
\label{blc_2}
\end{figure*}

\paragraph{Loss Landscape Visualization}
To better understand the relationships between the old model $\theta_{t-1}^*$, the incremental model $\theta_{t}$, the IVT model $\hat{\theta}_{t}$, and the oracle $\theta_{t}^*$, we visualize the test loss landscape in the parameter vector space, following \cite{Mirzadeh2020LinearMC, verwimp2021rehearsal}. As shown in Figure ~\ref{lsp}, the IVT model $\hat{\theta}_{t}$ remains significantly closer to the oracle $\theta_{t}^*$ compared with the incremental solution $\theta_t$, indicating a better alignment with the optimal solution across tasks.
The visualization reveals that $\theta_{t-1}^*$, $\hat{\theta}_{t}$, and $\theta_{t}^*$ are located within a connected low-loss basin, suggesting that the IVT update effectively preserves performance on old tasks. In contrast, $\theta_t$ veers into regions with elevated test loss, implying that standard incremental updates induce a shift toward suboptimal configurations with limited generalization across tasks.

\paragraph{Analysis of Training Time}
To evaluate the computational efficiency of IVT, we conduct a time complexity analysis to measure its training overhead. As shown in Figure ~\ref{fig:int} (right), IVT introduces only a negligible increase in training time compared to the baseline, indicating that its lightweight transformation step imposes minimal computational burden.
For a more comprehensive comparison, we benchmark IVT against EOPC, which performs an additional optimization phase after each task to explicitly search for a low-loss path. This extra stage significantly increases training costs, with EOPC requiring approximately $0.5\times$ to $1.6\times$ more training time than the baseline.
These results highlight the high computational efficiency of IVT. Unlike methods that rely on costly post-hoc optimization, IVT achieves strong performance gains with minimal time overhead, making it highly practical for real-world continual learning scenarios where both accuracy and efficiency are essential.

\begin{figure}[t]
\centering  
\subfloat{\includegraphics[width=0.47\linewidth]{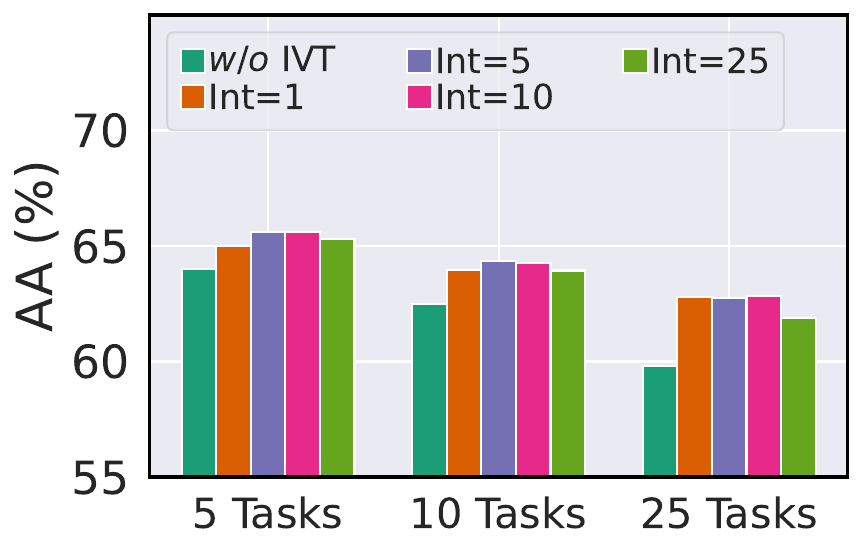}} \hfill
\subfloat{\includegraphics[width=0.50\linewidth]{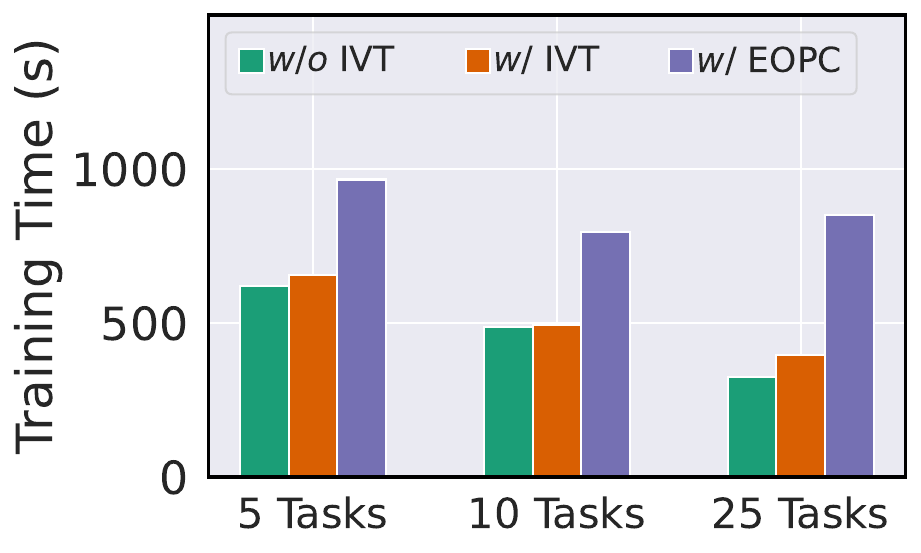}} 
  \caption{Analysis of the IVT interval (left) and training time for individual tasks (right) conducted with PODNet on CIFAR-100.} 
  \label{fig:int}
\end{figure}

\paragraph{The Effect of IVT Interval}
The interval between consecutive IVT applications is critical—too infrequent may lead to poor alignment with the loss landscape, while too frequent could interfere with task learning.
Hence, identifying a suitable interval is key to balancing transformation accuracy and connectivity.
To evaluate this, we conduct a sensitivity analysis on the IVT interval, which is notably the sole hyper-parameter of the method. As shown in Figure ~\ref{fig:int} (left), IVT demonstrates strong robustness across a wide range of interval settings. Regardless of the specific interval value, IVT consistently yields improvements over the baseline, underscoring its stability and reliability in practical deployment.
These results suggest that IVT does not require precise tuning to be effective, making it an attractive choice for continual learning scenarios where hyper-parameter optimization may be costly or infeasible. The broad performance plateau further indicates that IVT is resilient to scheduling variations, reinforcing its practicality and generality across diverse settings.

\begin{figure*}[t]
\centering
\hspace{10pt}\subfloat{\includegraphics[width=.3\linewidth]{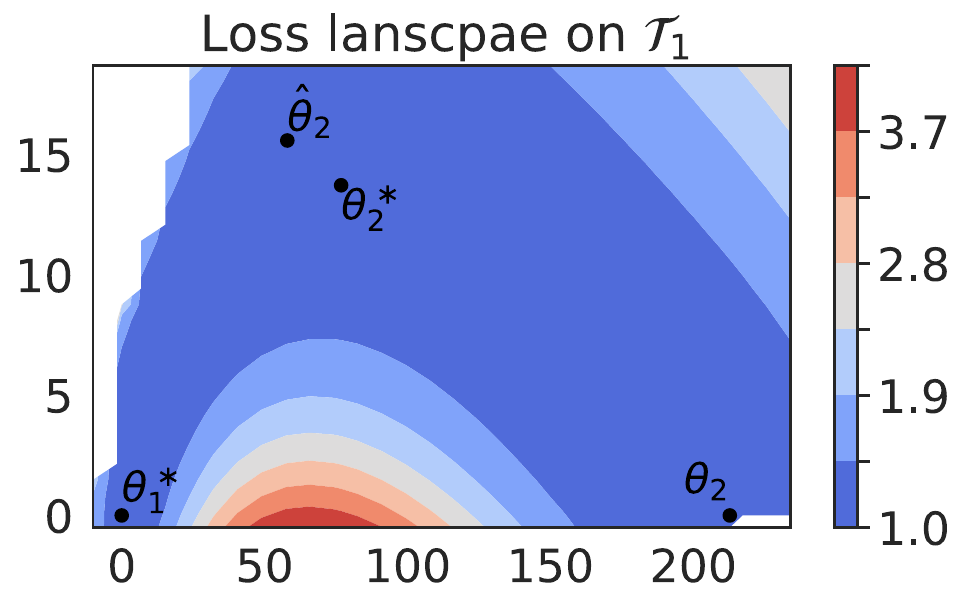}} \hfill
\subfloat{\includegraphics[width=.3\linewidth]{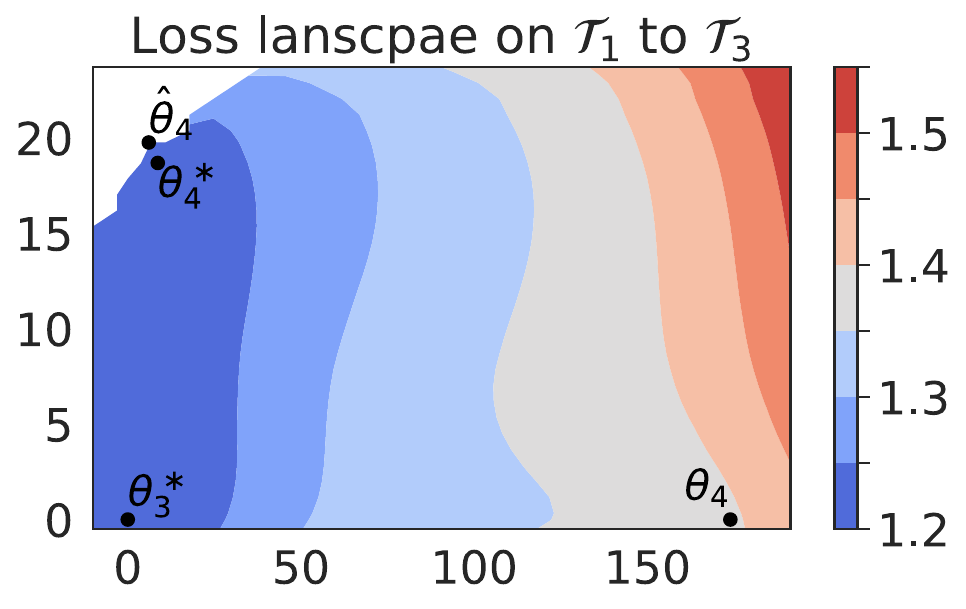}} \hfill
\subfloat{\includegraphics[width=.3\linewidth]{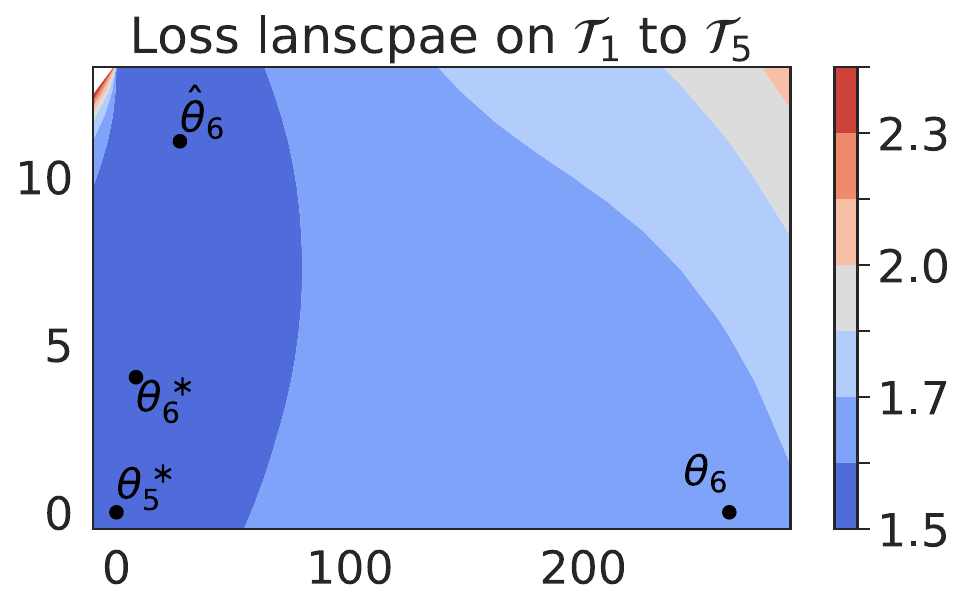}}\hspace{10pt}  
\caption{Visualization of the test loss landscape in the parameter vector space, produced by PODNet on CIFAR-100 with 5-task increments. The vector space is defined by orthogonalizing $\hat{\theta}_t - \theta^*_{t-1}$ and $\theta_t - \theta^*_{t-1}$, with the PODNet oracle $\theta^*_t$ as the projection onto this plane.}   
\label{lsp}
\end{figure*}

\paragraph{The Effect of Memory Size}
IVT does not rely on exemplar replay, making it broadly applicable to both exemplar-based and non-exemplar continual learning methods. To further assess its adaptability under constrained memory settings, we evaluate IVT on exemplar-based baselines using varying memory sizes, and compare its performance with the  path optimization method EOPC.
As shown in Table~\ref{tab:exem}, IVT consistently improves the baseline performance across different memory budgets. Notably, in the low-exemplar regime, where the risk of catastrophic forgetting is amplified due to insufficient rehearsal signals, IVT delivers substantial gains over the original model, demonstrating its capacity to preserve prior knowledge even with minimal supervision from exemplars.

As the memory size increases, the performance gap between IVT and EOPC narrows, with IVT achieving comparable results despite its simpler and more generalizable formulation. This observation highlights the robustness of IVT under strict memory constraints and suggests that its benefits are complementary to exemplar replay mechanisms. In scenarios where memory is severely limited or unavailable, IVT offers a compelling alternative for maintaining stability and accuracy in continual learning.

\begin{table}[t]
	\centering
 	\caption{Ablating memory size $|\mathcal{M}|$ under 5 tasks on CIFAR-100.} 
    \setlength{\tabcolsep}{6pt}
	\resizebox{1\linewidth}{!}{
		\begin{tabular}{l|cc|cc|cc}
			\toprule[1pt]
			    \multirow{2}{*}{Method} & \multicolumn{2}{c|}{$|\mathcal{M}|=5$}  & \multicolumn{2}{c|}{$|\mathcal{M}|=10$}  & \multicolumn{2}{c}{$|\mathcal{M}|=20$}  \\ 
& {$AA \uparrow$}    &  {$FM \downarrow$} &  {$AA \uparrow$}    &  {$FM \downarrow$}  &  {$AA \uparrow$}    &  {$FM \downarrow$} \\
\midrule
PODNet          & 54.63 & 26.06 & 61.28  & 21.82 & 64.00  & 17.72\\ 
 \hspace{3pt} \textit{w/} EOPC & 55.91 & 15.57 & 62.90  & 11.11 & 65.06  & 8.93\\ 
 \rowcolor{ cyan!5} \hspace{3pt} \textit{w/} IVT (Ours) & 62.04 & 17.11 & 63.02  & 13.06 & 65.36 & 11.68\\ 
			\bottomrule[1pt] 
        \end{tabular}}   
  \label{tab:exem} 
\end{table}

\paragraph{Comparison of Full and Diagonal FIM}
To balance effectiveness and efficiency, we compare the diagonal and full Fisher Information Matrix (FIM) approximations. The diagonal FIM, by assuming independence between parameters, enables fast computation and memory usage comparable to standard baseline methods.
To explore whether richer parameter dependencies improve performance, we further evaluate a full FIM variant restricted to convolutional kernels, capturing intra-channel parameter interactions while avoiding the full cost of global FIM computation. As shown in Table~\ref{tab:off}, incorporating the full FIM yields only marginal accuracy improvements across tasks. However, this comes at a significant cost: GPU memory usage increases by approximately 2.5$\times$, and computational time also grows notably due to the dense matrix operations.
This trade-off highlights the practicality of the diagonal FIM, which offers a favorable balance between performance and scalability. Particularly in large-scale or resource-constrained continual learning settings, the diagonal approximation remains the preferred choice, delivering strong results with minimal computational overhead.

\begin{table}[t]
\centering 
\caption{Comparison of IVT variant on CIFAR-100.}
\setlength{\tabcolsep}{4pt}
\resizebox{1\linewidth}{!}{
\begin{tabular}{l|ccc|ccc}
\toprule[1pt]
\multirow{2}{*}{Method} & \multicolumn{3}{c|}{5 tasks} &  \multicolumn{3}{c}{10 tasks}  \\
 & AA($\uparrow$) & LA($\uparrow$) & FM($\downarrow$) & AA($\uparrow$) & LA($\uparrow$) & FM($\downarrow$)  \\
\midrule
PASS & 58.60 & 49.05 & 8.36 & 56.55 & 44.95 & 11.75  \\
\rowcolor{cyan!5}  \hspace{3pt} \textit{w/} IVT (Diagonal) & 59.29 & 49.75 & 8.14 & 58.04 & 46.78 & 11.29\\
\rowcolor{cyan!5} \hspace{3pt}  \textit{w/} IVT (Full)& 59.73 & 49.92& 7.89 & 58.70 & 47.18 & 11.00 \\
\bottomrule[1pt]
\end{tabular}}
\captionsetup{skip=3pt}
\label{tab:off}
\end{table}

\section{Conclusion and Future Work}
In this paper, we investigate the linear mode connectivity (LMC) of CIL oracles and show that models can retain performance on earlier tasks by following low-loss linear paths in parameter space. Motivated by this insight, we propose Increment Vector Transformation (IVT), a lightweight plug-in that geometrically adjusts model updates to preserve such connectivity. IVT transforms the increment vector using diagonal Fisher Information Matrices and is applied periodically during training. This strategy improves new task learning with minimal interference to prior knowledge, and integrates seamlessly into a wide range of CIL algorithms with minimal overhead.
Extensive experiments on CIFAR-100, FGVCAircraft, ImageNet-Subset, and ImageNet-Full confirm the broad effectiveness of IVT. It consistently enhances performance across both exemplar-free and exemplar-based methods, and demonstrates strong compatibility with both scratch-trained and pre-trained learning paradigms. These results highlight IVT’s versatility and practical utility in real-world continual learning scenarios.

While IVT achieves a favorable balance between performance and efficiency, it currently adopts a diagonal FIM approximation to maintain scalability. Although this choice is sufficient in most cases, future work may further explore adaptive or structured approximations to better capture complex parameter interactions where needed.


\bibliographystyle{IEEEtran}
\bibliography{main}

\end{document}